%% file: autosam.tex
\documentclass[twocolumn,amsthm]{cls/autart}    % Enable this line and disable the 
\usepackage[table,usenames,dvipsnames]{xcolor} % color
\usepackage[noadjust]{cite}                    % Citation
\usepackage{amsmath,amssymb,amsfonts,amscd,dsfont}
\usepackage{algorithm,algorithmicx,listings}        % algorithms
\usepackage[noend]{algpseudocode}			        % necessary for algorithmicx
\usepackage{textcomp}
\usepackage{cls/slashbox}
\usepackage{graphicx}
\usepackage{tabularx, booktabs}
\newcolumntype{Y}{>{\centering\arraybackslash}X}
\usepackage{subcaption}
\usepackage{multirow,array} % multicol,
\usepackage[font={small}]{caption}
\usepackage{arydshln}
\usepackage{makecell}

\usepackage[colorinlistoftodos,prependcaption,textwidth=1.6cm,textsize=tiny]{todonotes}
\makeatletter
\newcommand{\thickhline}{%
    \noalign {\ifnum 0=`}\fi \hrule height 1pt
    \futurelet \reserved@a \@xhline
}
\newcolumntype{"}{@{\hskip\tabcolsep\vrule width 2pt\hskip\tabcolsep}}
\makeatother

\newcounter{SystemCounter}
\newtheorem{system}[SystemCounter]{System}

\usepackage[breaklinks=true, colorlinks, bookmarks=true, citecolor=Black, urlcolor=Violet,linkcolor=Black]{hyperref}

\algrenewcommand\algorithmicindent{0.65em}%

\renewcommand{\vec}[1]{\boldsymbol{\mathbf{#1}}}
\newcommand{\mat}[1]{\boldsymbol{\mathbf{#1}}}

\newcommand{\ceil}[1]{\lceil#1\rceil}

\newcommand{\scaleMathLine}[2][1]{\resizebox{#1\linewidth}{!}{$\displaystyle{#2}$}}
\newcommand{\prl}[1]{\mathopen{}\left(#1\right)\mathclose{}}

\newcommand{\crl}[1]{\mathopen{}\left\{#1\right\}\mathclose{}}

\usepackage{stackengine}
\stackMath

\usepackage{graphicx}          % Include this line if your 
\begin{document}

\begin{frontmatter}

\title{Large Scale Model Predictive Control with Neural Networks and Primal Active Sets\thanksref{footnoteinfo}} % Title, preferably not more 

\thanks[footnoteinfo]{This paper was not presented at any IFAC 
meeting. Corresponding author S.~W.~Chen. Tel. +XXXIX-VI-mmmxxi. 
Fax +XXXIX-VI-mmmxxv.}

\author[University of Pennsylvania]{Steven W. Chen}\ead{chenste@seas.upenn.edu},    % Add the
\author[University of California San Diego]{Tianyu Wang}\ead{tiw161@eng.ucsd.edu},   
\author[University of California San Diego]{Nikolay Atanasov}\ead{natanasov@eng.ucsd.edu},               % e-mail address 
\author[University of Pennsylvania]{Vijay Kumar}\ead{kumar@seas.upenn.edu},  % (ead) as shown
\author[University of Pennsylvania]{Manfred Morari}\ead{morari@seas.upenn.edu}  % (ead) as shown

\address[University of Pennsylvania]{GRASP Laboratory, University of Pennsylvania, Philadelphia, PA} 
\address[University of California San Diego]{Department of Electrical and Computer Engineering, University of California San Diego, La Jolla, CA}  % Please supply                                              
          % full addresses        % here.

                             % keyword list or with the 
                                          % help of the Automatica 
                                          % keyword wizard

\begin{abstract}                          % Abstract of not more than 200 words.
This work presents an explicit-implicit procedure to compute a model predictive control (MPC) law with guarantees on recursive feasibility and asymptotic stability. The approach combines an offline-trained fully-connected neural network with an online primal active set solver. The neural network provides a control input initialization while the primal active set method ensures recursive feasibility and asymptotic stability. The neural network is trained with a primal-dual loss function, aiming to generate control sequences that are primal feasible and meet a desired level of suboptimality. Since the neural network alone does not guarantee constraint satisfaction, its output is used to warm start the primal active set method online. We demonstrate that this approach scales to large problems with thousands of optimization variables, which are challenging for current approaches. Our method achieves a $2\times$ reduction in online inference time compared to the best method in a benchmark suite of different solver and initialization strategies.
% systems, whose states, inputs, and planning steps result in 
%These results are enabled by a data set generation algorithm that is critical for obtaining feasible training examples for high dimensional systems efficiently.
%This work presents an explicit-implicit procedure that combines an offline trained neural network with an online primal active set solver to compute a model predictive control (MPC) law with guarantees on recursive feasibility and asymptotic stability. The neural network improves the suboptimality of the controller performance and accelerates online inference speed for large systems, while the primal active set method provides corrective steps to ensure feasibility and stability. We introduce a primal-dual loss function to train a neural network to initialize the online controller. We then demonstrate online computation of the primal feasibility and suboptimality criteria to provide the desired guarantees. Next, we use these neural network and criteria measures to accelerate an online primal active set method through warm starts and early termination. Finally, we present a data set generation algorithm that is critical for efficiently generating a training set in high dimensional systems. The primary motivation is developing an algorithm that scales to systems that are challenging for current approaches, involving state and input dimensions as well as planning horizons in the order of tens to hundreds.
%, and we demonstrate high performance of our approach on a system with $36$ states, $9$ inputs, and a time horizon of $50$.
\end{abstract}

\end{frontmatter}

\input{tex/Introduction.tex}

\input{tex/Preliminaries.tex}

\input{tex/NeuralNetworkAndTraining.tex}
\input{tex/FeasibilityAndStability.tex}
\input{tex/DatasetGeneration.tex}

\input{tex/Examples.tex}
\input{tex/Conclusion.tex}

\begin{ack}                               % Place acknowledgements
This work is supported in part by ARO grant W911NF-13-1-0350, DARPA grant HR001151626/HR0011516850, NSF CRII RI IIS-1755568, the Semiconductor Research Corporation (SRC), and the NVIDIA AI Labs program. We are grateful to Elizabeth Wong for answering our many questions about SQOPT and providing guidance on initialization and accurate timing.% here.
\end{ack}

\bibliographystyle{abbrv}        % Include this if you use bibtex 
\bibliography{bib/autosam}       % and a bib file to produce the 
                                 % bibliography (preferred). The
                                 % correct style is generated by
                                 % Elsevier at the time of printing.

%\begin{thebibliography}{99}     % Otherwise use the  
                                 % thebibliography environment.
                                 % Insert the full references here.
                                 % See a recent issue of Automatica 
                                 % for the style.
%  \bibitem[Heritage, 1992]{Heritage:92}
%     (1992) {\it The American Heritage. 
%     Dictionary of the American Language.}
%     Houghton Mifflin Company.
%  \bibitem[Able, 1956]{Abl:56}
%     B.~C.~Able (1956). Nucleic acid content of macroscope. 
%     {\it Nature 2}, 7--9. 
%  \bibitem[Able {\em et al.}, 1954]{AbTaRu:54}   
%     B.~C. Able, R.~A. Tagg, and M.~Rush (1954).
%     Enzyme-catalyzed cellular transanimations.
%     In A.~F.~Round, editor, 
%     {\it Advances in Enzymology Vol. 2} (125--247). 
%     New York, Academic Press.
%  \bibitem[R.~Keohane, 1958]{Keo:58}
%     R.~Keohane (1958).
%     {\it Power and Interdependence: 
%     World Politics in Transition.}
%     Boston, Little, Brown \& Co.
%  \bibitem[Powers, 1985]{Pow:85}
%     T.~Powers (1985).
%     Is there a way out?
%     {\it Harpers, June 1985}, 35--47.

%\end{thebibliography}

%\appendix 
%\section{Appendix}
                                        % in the appendices.
\end{document}

%% file: tex/Introduction.tex
\section{Introduction}

Model predictive control (MPC) is a dynamic optimization technique widely used in industrial process applications, such as oil refineries and chemical plants \cite{qin2003survey}. Recently, MPC has found mainstream use in robotics for controlling ground \cite{richter2018bayesian}, aerial \cite{Watterson2015SafeRH,bouffard2012learning}, humanoid \cite{erez2013integrated}, and other autonomous robots due to its versatility, robustness, and safety guarantees. The transition from the process industry to robotics brings additional challenges since the available computation time is reduced from hours to milliseconds.

MPC techniques can be categorized into \emph{implicit} and \emph{explicit}. Implicit MPC focuses on online computation of an open-loop control sequence, optimizing the system performance at its current state. For example, Wang and Boyd~\cite{wang2010fast} exploit the structure of the quadratic program (QP) associated with the MPC problem to design efficient interior point methods.  Vichik and Borelli~\cite{vichik2014solving} experiment with a different computation scheme and demonstrate that analog circuits can solve QPs in microseconds through the use of custom hardware. Rather than solving an optimization problem online, explicit methods manage the computational load by pre-computing a control law $\vec{u} = \vec{\mu}(\vec{x})$ offline, as a function of all feasible states $\vec{x}$, where $\vec{\mu}$ is known to be piecewise affine on polytopes determined by the system constraints \cite[Ch. 11]{borrelli2017predictive}. Computing an optimal explicit control law can become computationally intractable in large problems because the number of polytopic regions may grow exponentially (worst case) with the number of constraints. Even determining which region contains the current system state online can require nontrivial processing power or memory storage~\cite{kvasnica2012clipping}. One approach to address the computational challenges of explicit MPC control is to construct an approximate sub-optimal controller. Jones and Morari~\cite{jones2010polytopic} use a double-description method to build piecewise-affine approximations of the value function and use barycentric functions on the polytopic regions to obtain an associated control law. The authors prove recursive feasibility and asymptotic stability of the suboptimal controller. Despite various approximation techniques, explicit MPC has not been demonstrated to scale to the same problem sizes as those handled by implicit MPC methods.

This paper investigates the use of a Rectified Linear Unit (ReLU) fully connected neural network to approximate the piecewise affine explicit MPC control law. There have been several recent works using neural networks for MPC design. Chen et al. \cite{Chen2018ApproximatingEM} use a neural network with an orthogonal projection operation to approximate the optimal control law. Hertneck et al. \cite{hertneck2018learning} use a neural network in a robust MPC framework to provide statistical guarantees of feasibility and stability. Zhang et al. \cite{zhang2019safe} use a neural network to approximate the primal and dual variables, and provide statistical guarantees as well as certificates of suboptimality. Our work extends these approaches by providing deterministic guarantees on recursive feasibility and asymptotic stability through corrective steps generated by an online QP solver. Moreover, we demonstrate for the first time that neural-network-based approximations of MPC control laws scale to large problems with thousands of optimization variables.

A few closely related works combine the strengths of machine learning for scalability with the analytical tractability of QP problems in an explicit-implicit MPC approach. Zeilinger et al. \cite{zeilinger2011real} compute a piecewise affine control law approximation, combine it with an active set method, and provide criteria to terminate the active set method early while still obtaining guarantees on recursive feasibility and asymptotic stability. Klau\v{c}o et al. \cite{klauvco2019machine} use classification trees and nearest neighbors to warm start an active set method and solve to the optimal solution. These approaches differ from our method because they do not utilize a neural network as the control policy approximator. In addition, they do not demonstrate nor address the challenges in scaling to large systems, specifically related to efficient generation of train data sets in high dimensions.

Generating large data sets of initial feasible states, optimal inputs, primal, and dual variables for high-dimensional systems to approximate an MPC controller has not been studied widely in the literature. Yet, a large data set is critical for accurate learning-based control law approximations, especially for high-dimensional systems. A na{\"i}ve rejection sampling approach to generate feasible initial states has rapidly diminishing probability of success as the number of constraints and dimensions grow. Existing works use gridding~\cite{klauvco2019machine,summers2011multiresolution} or random sampling~\cite{zhang2019safe, hertneck2018learning} methods that do not scale to high dimensions, or rely on reinforcement and imitation learning techniques that perform closed-loop simulations leading to weaknesses discussed in~\cite{ross2011reduction}. We propose an algorithm to efficiently generate large data sets in high dimensions using geometric random walks.

In summary, the goal of this work is to develop a neural-nework approximation of an MPC control law that scales to large systems, performs faster than implicit methods online, and provides guarantees on recursive feasibility and asymptotic stability. Our contributions include:
\begin{itemize}
	\item a primal-dual loss function that incorporates state and input constraints to train a deep ReLU network approximation of an explicit MPC control law (Sec.~\ref{sec:neural_network}); 
	\item an online approach that corrects the network output using a primal active set solver to guarantee recursive feasibility and asymptotic stability (Sec.~\ref{sec:guarantees_certificates});
	% and terminates early once sufficient criteria for recursive feasibility and asymptotic stability have been reached; \
	\item a geometric random walk algorithm for generating feasible sample data sets efficiently, necessary for training control law approximations for large systems (Sec.~\ref{sec:space_filling});
	%basutilizing ideas from geometric random walks that efficiently generates feasible samples for large systesms; and 
	\item a $2\times$ reduction in online inference time against the best benchmark method on a system with thousands of optimization variables. (Sec.~\ref{sec:results}).
\end{itemize}

%% file: tex/Preliminaries.tex
\section{Preliminaries}
\label{sec:preliminaries}

Consider a discrete-time linear time-invariant system,
\begin{equation}
\label{eq:system}
\vec{x}(t+1) = \mat{A}\vec{x}(t) + \mat{B}\vec{u}(t),
\end{equation}
subject to state and input constraints,
\begin{equation}
\begin{aligned}
\label{eq:constraints}
\vec{x}(t) \in \mathcal{X} &:= \{\vec{x}\in \mathbb{R}^{n}|\mat{A}_{x}\vec{x} \leq \vec{b}_{x} \} \quad &\forall~t& \geq 0,\\
\vec{u}(t) \in \mathcal{U} &:= \{\vec{u}\in \mathbb{R}^{m}|\mat{A}_{u}\vec{u} \leq \vec{b}_{u} \} \quad &\forall~t& \geq 0.
\end{aligned}
\end{equation}
Assuming that the pair $(\mat{A},\mat{B})$ is stabilizable \cite[Ch. 7]{borrelli2017predictive}, we are interested in obtaining a receding horizon controller (RHC) \cite[Ch. 12]{borrelli2017predictive}. At each time $t$, we solve a constrained finite-time optimal control problem \cite[Ch. 7]{borrelli2017predictive},
\begin{align}
& \min_{\vec{u}_{0:N-1}}  
& & J(\vec{u}_{0:N-1} | \vec{x}(t)) \stackunder[5pt]{{}= \vec{x}_{N}^{\top}\mat{P}\vec{x}_{N} {}}{ {}+ \sum_{k=0}^{N-1} \prl{\vec{x}_{k}^\top\mat{Q}\vec{x}_{k} +\vec{u}_{k}^\top\mat{R}\vec{u}_{k}}}\notag\\
& \quad\text{s.t.}
& & \vec{x}_{k+1} = \mat{A}\vec{x}_k +\mat{B}\vec{u}_k, \;\vec{x}_{0} = \vec{x}(t),\notag\\
& & & \mat{A}_{x}\vec{x}_k \leq \vec{b}_{x}, \mat{A}_{u}\vec{u}_k \leq \vec{b}_{u}, \mat{A}_{f}\vec{x}_{N} \leq \vec{b}_{f},\label{eq:time_invariant_finite_horizon_problem}
\end{align}
where symmetric matrices $\mat{Q}, \mat{P} \in \mathcal{S}^n_{\succ 0}$, $\mat{R}\in\mathcal{S}^m_{\succ 0}$ define desired system behavior over planning horizon $N$. The terminal cost, $\vec{x}_{N}^{\top}\mat{P}\vec{x}_{N}$, and terminal constraints, $\mathcal{X}_{f} := \{\vec{x} \in \mathbb{R}^{n} | \mat{A}_{f}\vec{x} \leq \vec{b}_{f}\}$, ensure feasibility and asymptotic stability of the RHC as will be discussed in Sec.~\ref{sec:rhc}. The first input, $\vec{u}_0$, is applied to the system, leading to a new state $\vec{x}(t+1)$, and the process is repeated. Alternatively, instead of recomputing $\vec{u}_{0:N-1}$ at every $t$, it may be desirable to obtain a function $\vec{u}_0 = \mu(\vec{x})$ that specifies the input for an arbitrary state $\vec{x}$.

\subsection{Batch Formulation}
The batch formulation of the problem in \eqref{eq:time_invariant_finite_horizon_problem} is:
\begin{equation}
\label{eq:decoupled_batch_problem}
\begin{aligned}
& \min_{\vec{z}}  
& & J(\vec{z}|\vec{x}) =  \vec{z}^{\top}\mat{H}\vec{z} + \vec{x}^{\top}\mat{Q}\vec{x}\\
& \;\;\text{s.t.}
& & \mat{G}_{\textrm{eq}}\vec{z} = \mat{E}_{\textrm{eq}}\vec{x}, \;\mat{G}_{\textrm{in}}\vec{z} \leq \vec{w}_{\textrm{in}} + \mat{E}_{in}\vec{x}
\end{aligned}
\end{equation}
with the following definitions
\begin{align}
\vec{z} &= [\vec{x}^{\top}_{1} \dots \vec{x}^{\top}_{N} ~ \vec{u}^{\top}_{0} \dots \vec{u}^{\top}_{N-1}] \in \mathbb{R}^{N(n+m)}\notag\\
\mat{H} &= \mathbf{diag}( \mat{I}_{N-1} \otimes \mat{Q}, \mat{P}, \mat{I}_{N}\otimes\mat{R} ) \in \mathbb{R}^{N(n+m) \times N(n+m)}\notag\\
\mat{G}_{\textrm{eq}} &= [\mat{I}_{Nn} - \mat{L}_N \otimes \mat{A} ; -\mat{I}_N \otimes \mat{B}] \in \mathbb{R}^{Nn\times N(n+m)}\notag\\
%\mat{G}_{\text{eq}} &= \left[
%\begin{array}{ccccc:cccccc}
%\mat{I} & ~ & ~ & ~ & ~ &  -\mat{B} & ~ & ~ & ~ & ~ &\\
%-\mat{A} & \mat{I} & ~ &  ~ & ~ & ~ & -\mat{B} & ~ & ~ & ~ &\\
%~ & ~ & \ddots & \ddots &  ~ & ~ & ~ & ~ & \ddots &\\
%~ & ~ & ~ & -\mat{A} & \mat{I} ~ & ~ & ~ & ~ &  ~ & -\mat{B}\\
%\end{array} \right ] \in \mathbb{R}^{Nn\times N(n+m)}\\
\mat{E}_{\textrm{eq}} &= \vec{e}_{1}\otimes \mat{A}\in \mathbb{R}^{Nn\times n} \label{eq:decoupled_constraint_matrices}\\
\mat{G}_{\textrm{in}} &= \mathbf{diag}(\vec{0}_{c_x},\mat{I}_{N-1}\otimes \mat{A}_{x},\mat{A}_{f}, \mat{I}_{N}\otimes \mat{A}_{u})\notag\\
\mat{E}_{\textrm{in}} &= -\vec{e}_{1}\otimes \mat{A}_{x} \in \mathbb{R}^{(Nc_x+c_f+Nc_u)\times n}\notag\\
\mat{w}_{\textrm{in}} &= [\vec{1}_{N} \otimes \vec{b}_{x}; \vec{b}_{f}; \vec{1}_N \otimes \vec{b}_{u}] \in \mathbb{R}^{(Nc_x + c_f + Nc_u)}\notag 
\end{align}
where $\mat{L}_N$ is the matrix of size $N \times N$ with ones on the first subdiagonal and zeros elsewhere, $\vec{e}_{i}$ is the $i$-th standard basis vector, $\vec{1}_N$ is the vector of all ones of size $N$, $c_{x}$, $c_{f}$, $c_{u}$ denote the number of constraints specified by the rows of $\mat{A}_{x}$, $\mat{A}_{f}$, $\mat{A}_{u}$, "$;$" denotes vertical concatenation, and $\otimes$ denotes the Kronecker product. When the parameter $\vec{x}$ is fixed, \eqref{eq:decoupled_batch_problem} is a \textit{quadratic program} (QP) and the solution is a vector $\vec{z}$. 

A common choice for the terminal cost matrix $\mat{P}$ is the solution $\mat{P}_{\infty}$ to the algebraic Riccati equation \cite[Ch. 8]{borrelli2017predictive}:
\begin{equation} 
\label{eq:dare}
\scaleMathLine[0.9]{\mat{P}_{\infty} = \mat{A}^{\top}\mat{P}_{\infty} \mat{A}+\mat{Q}-\mat{A}^{\top}\mat{P}_{\infty} \mat{B}(\mat{B}^{\top}\mat{P}_{\infty} \mat{B}+\mat{R})^{-1}\mat{B}^{\top}\mat{P}_{\infty} \mat{A}.}
\end{equation}
A common choice for $\mathcal{X}_{f}$ is the maximal positively invariant set, $\mathcal{O}^{LQR}_{\infty}$, of the linear quadratic regulator (LQR), which is computable via reachability analysis \cite[Ch.11]{borrelli2017predictive} using standard toolboxes \cite{MPT3}. Note that $\mathcal{O}^{LQR}_{\infty}$ is easily computable even for large systems, and is different from the maximal control invariant set, $\mathcal{C}^{\infty}$, used in \cite{Chen2018ApproximatingEM}, which is very challenging to compute for large systems.

\subsection{Feasibility and Duality}
\label{sec:feasibility_and_duality}

%These positive definite assumptions ensure strict convexity of the resulting optimization problems. 
For a given state $\vec{x}$, \eqref{eq:decoupled_batch_problem} is a strictly convex QP because if $\mat{Q}$, $\mat{P}$ are positive definite then $\mat{H}$ is positive definite too. While this assumption can be relaxed to allow positive semi-definite matrices, it simplifies the development of our approach for approximating the solution of \eqref{eq:decoupled_batch_problem}, which is unique in this case. Let $\mathcal{X}_0 \subseteq \mathcal{X}$ be the set of parameters $\vec{x}$ for which \eqref{eq:decoupled_batch_problem} is feasible~\cite[Defn. 6.3]{borrelli2017predictive}. Given $\vec{x} \in \mathcal{X}_0$ and an associated primal feasible $\vec{z}$, the \textit{suboptimality level} of $\vec{z}$ is:
%For a given state $\vec{x}$, \eqref{eq:decoupled_batch_problem} is a strictly convex QP because the matrix $\mat{H}$ is positive definite. Let $\mathcal{X}_0 \subseteq \mathcal{X}$ be the set of parameters $\vec{x}$ for which \eqref{eq:decoupled_batch_problem} is feasible~\cite[Defn. 6.3]{borrelli2017predictive}. Given $\vec{x} \in \mathcal{X}_0$ and an associated primal feasible $\vec{z}$, the \textit{suboptimality level} of $\vec{z}$ is:
\begin{equation}
\label{eq:suboptimality_level}
\sigma(\vec{z}|\vec{x}) := J(\vec{z}|\vec{x}) - J(\vec{z}^*|\vec{x}) \geq 0,
\end{equation}
where $\vec{z}^*$ is the unique minimizer of \eqref{eq:decoupled_batch_problem}. We introduce \textit{dual variables} $(\vec{\nu}, \vec{\lambda})$~\cite[Ch.5]{boyd2004convex}, and define the Lagrangian associated with \eqref{eq:decoupled_batch_problem} as:
\begin{align}
\mathcal{L}(\vec{z}, \vec{\nu}, \vec{\lambda}| \vec{x}) := \vec{z}^{\top}\mat{H}&\vec{z} + \vec{x}^{\top}\mat{Q}\vec{x} + 
\vec{\nu}^{\top}\prl{\mat{G}_{\textrm{eq}}\vec{z} - \mat{E}_{\textrm{eq}}\vec{x}}\notag\\
& +\vec{\lambda}^{\top}\prl{\mat{G}_{\textrm{in}}\vec{z} - \vec{w}_{\textrm{in}} - \mat{E}_{in}\vec{x}}.\label{eq:lagrangian}
\end{align}
The Lagrangian dual of a minimization QP is a maximization QP over $(\vec{\nu}, \vec{\lambda})$ with objective function $d(\vec{\nu}, \vec{\lambda}| \vec{x}) := \inf_{\vec{z}} \mathcal{L}(\vec{z}, \vec{\nu}, \vec{\lambda}| \vec{x})$.
%\begin{equation}
%\label{eq:dual}
%\begin{aligned}
%d(\vec{\nu}, \vec{\lambda}| \vec{x}) &= \inf_{\vec{z}} \mathcal{L}(\vec{z}, \vec{\nu}, \vec{\lambda}| \vec{x}).
%\end{aligned}
%\end{equation}
The dual variables $(\vec{\nu}, \vec{\lambda})$ are \textit{dual feasible} if $\vec{\lambda} \geq 0$. For a given parameter $\vec{x}$, any primal feasible $\vec{z}$, any dual feasible $(\vec{\nu}, \vec{\lambda})$, and optimal primal-dual variables $(\vec{z}^{*}, \vec{\nu}^{*}, \vec{\lambda}^{*})$, \textit{strong duality} holds because Slater's conditions are satisfied~\cite[Ch.5]{boyd2004convex}:
\begin{equation}
\label{eq:strong_duality}
\begin{aligned}
d(\vec{\nu}, \vec{\lambda}| \vec{x}) \leq d(\vec{\nu}^{*}, \vec{\lambda}^{*}| \vec{x}) &= \mathcal{L}(\vec{z}^{*},\vec{\nu}^{*},\vec{\lambda}^{*}| \vec{x})\\
&= J(\vec{z}^{*}| \vec{x}) \leq J(\vec{z}| \vec{x}).
\end{aligned}
\end{equation}
The \textit{feasible duality gap} associated with $\vec{x}$, $\vec{z}$, $\vec{\nu}$, $\vec{\lambda}$ is
\begin{equation}
\label{eq:duality_bound_suboptimality}
\eta(\vec{z}, \vec{\nu}, \vec{\lambda}| \vec{x}) := J(\vec{z}| \vec{x}) - d(\vec{\nu},\vec{\lambda}|\vec{x}) \geq \sigma(\vec{z}|\vec{x}) \geq 0,
\end{equation}
which is an upper bound on the suboptimality level $\sigma(\vec{z}|\vec{x})$ for any feasible $(\vec{z}, \vec{\nu}, \vec{\lambda})$ due to \eqref{eq:strong_duality}.

\subsection{Receding Horizon Control}
\label{sec:rhc}

%Fixing $\vec{x}$ yields a QP in Prob.~\eqref{eq:decoupled_batch_problem} over the primal variables $\vec{z}$. 
Alternatively, \eqref{eq:decoupled_batch_problem} can be viewed as a \textit{multiparametric quadratic program} (mp-QP) over functions that map $\vec{x}$ to primal variables $\vec{z}$ \cite[Ch. 6]{borrelli2017predictive}.
\begin{defn}
A planner is a function $\pi: \mathcal{X}_{0} \rightarrow \mathbb{R}^{N(n+m)}$ that maps a parameter $\vec{x}$ to a decision variable $\vec{z}$. A planner $\pi$ is primal feasible if $\vec{z} = \pi(\vec{x})$ is a primal feasible variable $\forall\vec{x} \in \mathcal{X}_{0}$ in~\eqref{eq:decoupled_batch_problem}. It is optimal, if $\vec{z} = \pi(\vec{x})$ is the optimal solution $\forall\vec{x} \in \mathcal{X}_{0}$.
\end{defn}
%
%The following definitions show how a planner $\pi$ can be implemented as a receding horizon controller (RHC) $\mu$. 

%\begin{defn} (Receding Horizon Control)
%Let $\pi$ be an open loop mapping for optimization problem~\eqref{eq:decoupled_batch_problem}.  Define the receding horizon controller (RHC) $\mu:\mathbb{R}^{Nn+Nm} \rightarrow \mathbb{R}^{m}$ to be a linear mapping that returns the first control input $\vec{u}_{0}$ in $\vec{z} = \pi(\vec{x}(t))$. If $\pi$ is a primal feasible parametric solution, then $\mu$ is a primal feasible RHC.
%\end{defn}

%\begin{defn}[Closed Loop System and RHC]
\begin{defn}
The receding horizon controller (RHC) corresponding to a planner $\pi$ is a function $\mu:\mathcal{X}_{0} \rightarrow \mathbb{R}^{m}$ that returns the first control input in $\vec{z} = \pi(\vec{x})$, i.e., $\vec{u}_{0} = \mu(\vec{x})$.
\end{defn}

Recursive feasiblity and asymptotic stability \cite[Ch.12]{borrelli2017predictive} are two properties that should be guaranteed for an RHC. Primal feasibility is a property of a planner $\pi$ and the open-loop optimization problem, while recursive feasibility is a property of an RHC $\mu$ and the corresponding closed loop system $\vec{x}(t+1) = \mat{A}\vec{x}(t) + \mat{B} \mu(\vec{x}(t))$ for $t \geq 0$. In general, primal feasibility of the open-loop optimization problem does not imply recursive feasibility of the corresponding closed-loop controller. Recursive feasibility of a control law is a necessary, but not sufficient, condition for asymptotic stability.

%% file: tex/NeuralNetworkAndTraining.tex
\section{Explicit Control Policy Approximation}
%\section{Offline Neural Network Training}
\label{sec:neural_network}

The first step of our approach is offline training of a deep neural network $\tilde{\pi}(\vec{x}|\vec{\theta})$ with parameters $\vec{\theta}$ that provides a candidate solution $\vec{z}$ for the QP in \eqref{eq:decoupled_batch_problem}. We choose to approximate the entire primal prediction over the planning horizon instead of the control law $\mu$. Although this choice requires approximating $N(n+m)$ variables instead of only $m$, the additional predicted variables may be used to obtain the desired guarantees through an online primal active set solver in Sec.~\ref{sec:guarantees_certificates}.

\begin{defn}
A deep neural network (\textit{DNN}) $\tilde{\pi}(\vec{x}|\vec{\theta})$ with $L$ layers is a composition of $L$ affine functions: $\vec{\lambda}_l(\vec{x}) := \mat{\theta}^{W}_l\vec{x}+\vec{\theta}^{b}_{l}$,
each except the last one followed by a nonlinear activation function $\vec{h}$, so that: $\tilde{\pi}(\vec{x}|\vec{\theta}) = \vec{\lambda}_{L} \circ \vec{h} \circ \vec{\lambda}_{L-1} \circ \cdots \circ \vec{h} \circ \vec{\lambda}_{1}(\vec{x})$.
The DNN parameters are $\vec{\theta}:=\{(\mat{\theta}^{W}_{l}, \vec{\theta}^{b}_{l})\}_{l=1}^L$. Each layer $l$ has width defined by the number of rows of $\mat{\theta}^{W}_{l}$ and $\vec{\theta}^{b}_{l}$.
\end{defn}

The activation function $\vec{h}$ is fixed (not optimized) and typically chosen as a sigmoid, hypertangent, or ReLU~\cite[Ch.6]{Goodfellow-et-al-2016}. As noted in~\cite{Chen2018ApproximatingEM}, a ReLU activation $\vec{h}(\vec{x}) := \max\{\vec{0},\vec{x}\}$, where the $\max$ is applied elementwise, is well-suited for mp-QP problems because any piecewise-affine function on polyhedra, such as the solution for an mp-QP~\cite[Thm. 6.7]{borrelli2017predictive}, can be represented exactly by a ReLU DNN~\cite[Thm. 2.1]{arora2018understanding}. We restrict the approximation of a planner $\pi$ to the class of functions represented by a ReLU DNN $\tilde{\pi}$.

A train data set $\mathcal{D} := \crl{\vec{x}_{i}, \vec{z}_{i}^{*}, \vec{\nu}_{i}^{*}, \vec{\lambda}_{i}^{*}}_i$ of optimal primal and dual variables is necessary for supervised learning of the neural network parameters $\vec{\theta}$ that approximate an optimal planner $\pi^{*}(\vec{x}) \approx\tilde{\pi}(\vec{x}|\vec{\theta})$. Optimizing a least-squares loss function measuring the discrepancy between $\vec{z}_{i}^{*}$ and $\vec{z}_i := \tilde{\pi}(\vec{x}_i|\vec{\theta})$ is a common choice but it fails to explicitly account for the constraints in \eqref{eq:decoupled_batch_problem}. We propose a loss function which incorporates information from the optimal primal and dual variables by measuring the discrepancy of the Lagrangian values:
\begin{equation*}
%\label{eq:least_squares_Lagrangian}
\ell(\vec{\theta}) := \sum_{i=1}^{|\mathcal{D}|}(\mathcal{L}(\tilde{\pi}(\vec{x}_{i}|\vec{\theta}), \vec{\nu}_{i}^{*}, \vec{\lambda}_{i}^{*}| \vec{x}_{i}) - \mathcal{L}(\vec{z}_{i}^{*}, \vec{\nu}_{i}^{*}, \vec{\lambda}_{i}^{*}|\vec{x}_{i}))^{2}.
\end{equation*}
%
% ~\eqref{eq:least_squares_Lagrangian}
There are a variety of potential loss functions that utilize the Lagrangian. We chose this particular one because it allows supervision from both the optimal primal and dual variables and was effective in our experiments (Sec.~\ref{sec:results}).
%due to its simplicity and practical effectiveness. 
Since $\pi^*(\vec{x})$ is piecewise-affine, it can be represented exactly by $\tilde{\pi}(\vec{x}|\vec{\theta})$ with sufficient depth and width. By strict convexity, the primal variables $\vec{z}^{*}_{i}$ are unique minimizers of each term in the loss $\ell(\vec{\theta})$. We train the network using stochastic gradient descent \cite{kingma2014adam}, sampling a subset (mini-batch) of the train data $\mathcal{D}$, computing the loss, $\ell(\vec{\theta})$, over the mini-batch, and using backpropagation to compute the loss gradient and update the network parameters $\vec{\theta}$. The mini-batch sampling is continued until the data set is exhausted, which concludes an epoch, and is then repeated until convergence.

%% file: tex/FeasibilityAndStability.tex
\section{Primal Active Set Method for Guarantees}
\label{sec:guarantees_certificates}

% represents an approximate planner $\pi$ that maps states $\vec{x}$ to primal predictions $\vec{z}$.
% recomputing $\vec{z} = \tilde{\pi}(\vec{x}(t)|\vec{\theta})$ and  
The trained neural network $\tilde{\pi}(\vec{x}|\vec{\theta})$ is a planner that maps states $\vec{x}$ to primal predictions $\vec{z}$. It can be implemented as an RHC $\mu$ by applying the first control input $\vec{u}_{0}$ from $\vec{z} = \tilde{\pi}(\vec{x}(t)|\vec{\theta})$ at each encountered state $\vec{x}(t)$ over time. Thm.~\ref{thm:approximate_RHC_stability} summarizes the conditions that ensure recursive feasibility and asymptotic stability of $\mu$.

\begin{thm}
\label{thm:approximate_RHC_stability}
Let $\pi$ be a planner for the finite-horizon deterministic optimal control problem in~\eqref{eq:time_invariant_finite_horizon_problem} with $\vec{b}_{x},\vec{b}_{u},\vec{b}_{f} > \vec{0}$. Assume that the terminal constraint set $\mathcal{X}_{f}$ is control invariant and the terminal cost $p(\vec{x}) := \vec{x}^\top \mat{P} \vec{x}$ is a control Lyapunov function~\cite[Rmk 12.3]{borrelli2017predictive} over $\mathcal{X}_{f}$. If for all $\vec{x} \in \mathcal{X}_{0}$, $\vec{z} = \pi(\vec{x})$ is primal feasible and there exists a function $\gamma(\vec{x})$ such that:
\begin{equation}
\label{eq:gamma_conditions}
\begin{aligned}
0 \leq \sigma(\pi(\vec{x})| \vec{x}) \leq \gamma(\vec{x}) \leq \vec{x}^\top \mat{Q}\vec{x},
\end{aligned}
\end{equation}
then the RHC $\mu$ corresponding to $\pi$ is recursively feasible and asymptotically stable with domain of attraction $\mathcal{X}_{0}$ for system~\eqref{eq:system} subject to constraints~\eqref{eq:constraints}.
\end{thm}
%\vspace*{-3ex}
\begin{proof}
Recursive feasibility of $\mu$ follows from \cite[Thm.~12.1]{borrelli2017predictive}, while asymptotic stability follows from \cite[Thm.~13.1]{borrelli2017predictive}.
\end{proof}
%\vspace*{-1ex}
	
If we guarantee that the neural network output $\vec{z}(t) = \tilde{\pi}(\vec{x}(t)|\vec{\theta})$ at each time step $t$ is primal feasible and that there exists $\gamma(\vec{x}(t))$ that satisfies threshold~\eqref{eq:gamma_conditions}, then according to Thm.~\ref{thm:approximate_RHC_stability}, the RHC $\vec{u}(t) = \mu(\vec{x}(t))$ will be recursively feasible and asymptotically stable. In Sec.~\ref{sec:certificates_feasibility_suboptimality}, we show how to construct $\gamma(\vec{x})$ and check whether a given network output $\vec{z}$ satisfies primal feasibility and threshold~\eqref{eq:gamma_conditions}. Then, in Sec.~\ref{sec:active_set}, we use the network output $\vec{z}$ as initialization of an online primal active set solver which optimizes $\vec{z}$ until primal feasibility and the suboptimality certificate in~\eqref{eq:gamma_conditions} are both satisfied.

%%%========================================================================
\subsection{Obtaining a Suboptimality Certificate}
\label{sec:certificates_feasibility_suboptimality}

Given $\vec{z} = \tilde{\pi}(\vec{x}|\vec{\theta})$, primal feasibility can be checked using the constraints in \eqref{eq:decoupled_batch_problem}. Sec.~\ref{sec:active_set} details how to obtain primal feasibility given a primal infeasible initialization. Once $\vec{z}$ is primal feasible, we obtain dual feasible variables $(\vec{\nu}, \vec{\lambda})$ and define $\gamma(\vec{x})$ as the feasible duality gap $\eta(\vec{z}, \vec{\nu}, \vec{\lambda} | \vec{x})$ in \eqref{eq:duality_bound_suboptimality}. This guarantees that $\sigma(\vec{z} | \vec{x}) \leq \eta(\vec{z}, \vec{\nu}, \vec{\lambda} | \vec{x})$ and we can directly check whether $\eta(\vec{z}, \vec{\nu}, \vec{\lambda} | \vec{x}) \leq \vec{x}^\top \mat{Q}\vec{x}$ to satisfy \eqref{eq:gamma_conditions}.

We use the KKT optimality conditions to derive $(\vec{\nu}, \vec{\lambda})$. First, we determine which constraints in \eqref{eq:decoupled_batch_problem} are active at $\vec{z}$. All $Nn$ equality constraints $\mat{G}_{\textrm{eq}}\vec{z} = \mat{E}_{\textrm{eq}}\vec{x}$ are active by definition. Let $\mathcal{A}$ be the set of inequality constraints $i$ such that $\vec{e}_i^\top\mat{G}_{\textrm{in}}\vec{z} = \vec{e}_i^\top\vec{w}_{\textrm{in}} + \vec{e}_i^\top\mat{E}_{in}\vec{x}$. Due to \textit{complementary slackness}, dual variables corresponding to inactive constraints should be set to $0$. The remaining dual variables $(\vec{\nu}, \vec{\lambda}_{\mathcal{A}})$ must satisfy the KKT conditions for an equality constrained QP:
\begin{equation}
\label{eq:equality_constrained_KKT}
\begin{bmatrix}
2\mat{H} & \mat{G}_{\textrm{eq}}^\top & \mat{G}_{\mathcal{A}}^{\top}\\
\mat{G}_{\textrm{eq}} & \mat{0} & \mat{0}\\
\mat{G}_{\mathcal{A}} & \mat{0} & \mat{0}
\end{bmatrix} \begin{bmatrix} \vec{z}\\\vec{\nu}\\ \vec{\lambda}_{\mathcal{A}} \end{bmatrix} = 
\begin{bmatrix} \mat{0} \\ \mat{E}_{\textrm{eq}}\vec{x} \\ \vec{w}_{\mathcal{A}} +  \mat{E}_{\mathcal{A}}\vec{x} \end{bmatrix},
\end{equation}
where $\mat{G}_{\mathcal{A}}$, $\vec{w}_{\mathcal{A}}$, and $\mat{E}_{\mathcal{A}}$ contain the rows of $\mat{G}_{\textrm{in}}$, $\vec{w}_{\textrm{in}}$, and $\mat{E}_{in}$ corresponding to $\mathcal{A}$. The dual variables can be expressed directly in terms of $\vec{z}$:
\begin{equation}
\scaleMathLine[0.88]{\begin{bmatrix}
\mat{G}_{\textrm{eq}}\mat{H}^{-1} \mat{G}_{\textrm{eq}}^\top & \mat{G}_{\textrm{eq}}\mat{H}^{-1} \mat{G}_{\mathcal{A}}^\top\\
\mat{G}_{\mathcal{A}}\mat{H}^{-1} \mat{G}_{\textrm{eq}}^\top & \mat{G}_{\mathcal{A}}\mat{H}^{-1} \mat{G}_{\mathcal{A}}^\top
\end{bmatrix} \begin{bmatrix} \vec{\nu}\\ \vec{\lambda}_{\mathcal{A}} \end{bmatrix} = 
-2\begin{bmatrix}
\mat{G}_{\textrm{eq}} \\ \mat{G}_{\mathcal{A}}
\end{bmatrix} \vec{z}.}
\end{equation}
Finally, we can ensure that $\vec{\lambda}_{\mathcal{A}}$ are dual feasible by applying an element-wise $\max\{\vec{\lambda}_{\mathcal{A}}, \vec{0}\}$. Given a primal feasible $\vec{z}$, this always yields dual feasible $\vec{\nu}, \vec{\lambda}$ and we can obtain $\eta(\vec{z}, \vec{\nu}, \vec{\lambda} | \vec{x}) = J(\vec{z}| \vec{x}) - d(\vec{\nu}, \vec{\lambda} | \vec{x})$. The dual objective $d(\vec{\nu}, \vec{\lambda} | \vec{x})$ can be computed efficiently by pre-computing $\frac{1}{2}\mat{H}^{-1} [ \mat{G}_{\textrm{eq}}; \mat{G}_{\textrm{in}} ]^\top$ offline.

%The above procedure only checks whether primal feasibility and the suboptimality requirement in~\eqref{eq:duality_bound_suboptimality} are satisfied, but does not guarantee satisfaction. The predicted primal variables $\vec{z}$ may not necessarily be primal feasible and will not yield the certificate of primal feasibility. In addition, the procedure may not yield the certificate of suboptimality because the predicted primal variables or derived dual feasible variables are too conservative, and the resulting $\eta > q(\vec{x}, \vec{0})$. The following Sec.~\ref{sec:active_set} details how we guarantee satisfying both certificates $~\forall~\vec{x} \in \mathcal{X}_{0}$ by using an online primal active set method, and thus obtain the guarantees on recursive feasibility and asymptotic stability of the approximate RHC $\tilde{\mu}$.

%%%========================================================================

%\subsection{Obtaining Guarantees with Primal Active Set Solver}
\subsection{Ensuring Feasibility and Bounded Suboptimality with an Online Primal Active Set Solver}
\label{sec:active_set}

The procedure in Sec.~\ref{sec:certificates_feasibility_suboptimality} only checks whether the neural network output $\vec{z} = \tilde{\pi}(\vec{x}|\vec{\theta})$ is primal feasible and satisfies the suboptimality certificate in~\eqref{eq:duality_bound_suboptimality} but does not provide an approach to modify $\vec{z}$ in case any of the two requirements is violated. We use a primal active-set QP solver~\cite[Ch.~16]{wright1999numerical} online to ensure satisfaction of both conditions. Primal active-set QP methods can be accelerated via warm start from a good initialization point, such as the neural network output $\vec{z}$. 

A primal active-set method employs two phases. In \textit{Phase I}, a linear feasiblity program is solved to ensure that $\vec{z}$ is primal feasible. Starting from this primal feasible point, \textit{Phase II} updates the primal solution and the active constraints $\mathcal{A}$ and solves a sequence of equality-constrained QPs. Phase II maintains primal feasibility throughout the iterations and, hence, can be terminated early, as soon as the suboptimality certificate in~\eqref{eq:gamma_conditions} is achieved, rather than continuing until optimality. We check the duality gap at intermediate iterates using the approach described in Sec.~\ref{sec:certificates_feasibility_suboptimality} and terminate Phase II as soon as the suboptimality certificate is obtained. Compared to other QP solvers, our approach accelerates online computation by reducing the initial iterations using good initializations from the neural network. It also reduces the final iterations using the early termination condition in~\eqref{eq:gamma_conditions}. Interior-point methods are difficult to warm start~\cite{john2008implementation}, while non-primal active-set methods do not guarantee primal feasibility of intermediate iterates and cannot be terminated early~\cite{Ferreau2008}. In general, however, the neural network output can be used to accelerate any QP solver which benefits from a good but potentially infeasible initialization point.

Alg.~\ref{alg:explicit_implicit_planner} presents our final explicit-implicit planner $\pi$. It is guaranteed to terminate since primal active-set QP methods reach an optimal solution in finite time for strictly convex QPs~\cite[Ch.~16]{wright1999numerical} and the associated RHC $\mu$ is guaranteed to be recursively feasible and asymptotically stable by Thm.~\ref{thm:approximate_RHC_stability}.

\begin{algorithm}[t]
	\caption{Explicit-Implicit Planner}
	\label{alg:explicit_implicit_planner}
	\begin{algorithmic}[1]
	\small
	\Require initial state $\vec{x}$, NN parameters $\vec{\theta}$, primal active-set QP solver $\alpha$
	\Procedure{ExplicitImplicitPlanner}{$\vec{x}, \vec{\theta}, \alpha$}
		\State Obtain neural network prediction $\vec{z} = \tilde{\pi}(\vec{x}|\vec{\theta})$
		\If {$\vec{z}$ is not primal feasible } \Comment Sec.~\ref{sec:certificates_feasibility_suboptimality}
		  \State Perform \textit{Phase I} (feasibility) of $\alpha$
		\EndIf
		\While{\textbf{not} \Call{CertifySuboptimality}{$\vec{x},\vec{z}$} }
		  \State Update $\vec{z}$ via one \textit{Phase II} iteration of $\alpha$
		\EndWhile
		\State \Return $\vec{z}$
  \EndProcedure
  \Procedure{CertifySuboptimality}{$\vec{x},\vec{z}$}
    \State Obtain dual feasible variables $\vec{\nu}, \vec{\lambda}$ \Comment Sec.~\ref{sec:certificates_feasibility_suboptimality}
    \If{$\eta(\vec{z}, \vec{\nu}, \vec{\lambda} | \vec{x}) \leq \vec{x}^\top \mat{Q}\vec{x}$} \textbf{return} True \Comment Eqn.~\eqref{eq:duality_bound_suboptimality}
    \Else{ }\textbf{return} False
    \EndIf
  \EndProcedure

	\end{algorithmic}
\end{algorithm}

%% file: tex/DatasetGeneration.tex
\section{Scaling to Large Systems}
\label{sec:space_filling}

The efficiency of Alg.~\ref{alg:explicit_implicit_planner} online depends on the neural network $\tilde{\pi}(\vec{x}|\vec{\theta})$ accurately approximating an optimal planner $\pi^{*}(\vec{x})$ for the mp-QP in~\eqref{eq:decoupled_batch_problem}. If $\tilde{\pi}(\vec{x}|\vec{\theta})$ approximates $\pi^{*}(\vec{x})$ perfectly, Alg.~\ref{alg:explicit_implicit_planner} does not need to perform any computations online. However, if $\tilde{\pi}(\vec{x}|\vec{\theta})$ is a poor approximation of $\pi^{*}(\vec{x})$, Alg.~\ref{alg:explicit_implicit_planner} may need to perform many Phase I and Phase II iterations.

Providing a large train set $\mathcal{D} = \crl{\vec{x}_{i}, \vec{z}_{i}^{*}, \vec{\nu}_{i}^{*}, \vec{\lambda}_{i}^{*}}_i$ is important for good neural network performance and generalization. This data set is generated offline using a QP solver to obtain optimal solutions $(\vec{z}_{i}^*,\vec{\nu}_{i}^{*}, \vec{\lambda}_{i}^{*})$ for initial states $\vec{x}_{i}$ that render \eqref{eq:decoupled_batch_problem} feasible. The main challenge is that the set $\mathcal{X}_0$ of initial states that make \eqref{eq:decoupled_batch_problem} feasible cannot be described explicitly, e.g., via halfspace or vertex representation. This makes sampling $\vec{x}_{i}$ from $\mathcal{X}_0$ challenging, especially for high-dimensional systems with long planning horizons.

The set $\mathcal{X}_0$ can only be described by a membership oracle~\cite{vempala2005geometric}, i.e., for a given $\vec{x} \in \mathcal{X}$, a QP solver can report whether~\eqref{eq:decoupled_batch_problem} is feasible, i.e., $\vec{x} \in \mathcal{X}_0$ or not. Unfortunately, a simple approach, such as \emph{rejection sampling}, which generates samples $\vec{x}_i$ uniformly from $\mathcal{X}$ and only keeps those that are feasible according to the QP solver, cannot be used because the probability of sampling feasible $\vec{x}_i$ decreases quickly with the system dimension and the number of constraints in $\mathcal{X}$, $\mathcal{U}$, and $\mathcal{X}_f$.

Instead of independent sampling from $\mathcal{X}$, we propose a \emph{geometric random walk} technique which generates new samples based on previous successful samples from $\mathcal{X}_0$. The proposed approach starts with one feasible sample $\vec{x}_0 \in \mathcal{X}_0$, e.g., chosen as the system equilibrium state. We pick a random line $l$ at the current feasible point and iteratively take small steps along the chord $l \cap \mathcal{X}_{0}$. At each potential point $\vec{x}_i$ along $l$, we check whether the QP in~\eqref{eq:decoupled_batch_problem} with parameter $\vec{x}_i$ is feasible. If yes, the algorithm moves to $\vec{x}_i$ and proceeds along $l$, otherwise, it stays at $\vec{x}_{i-1}$ and picks a new random direction. To select random directions, we generate a set of goal states $\vec{x}_g \in \mathcal{X}$ using a Sobol sequence~\cite{sobol1967distribution}. Alg.~\ref{alg:space_filling_tree} summarizes the proposed approach for data set generation. It generates one large data set and then splits it into a train set, buffer set, and test set. Due to the sequential nature of the data set generation, the buffer set is needed to ensure that the train and test sets do not contain overlapping seed and goal points. Fig.~\ref{fig:DatasetGen} illustrates the behavior of Alg.~\ref{alg:space_filling_tree} for a double integrator system described in Sec.~\ref{sec:system_description}.

%Our method picks a random line $l$ at the current point $\vec{x}$ and goes to a random point on the chord $l \cap \mathcal{X}_{0}$. At each potential point $\vec{x}'$, the random walk will queries the membership oracle and moves to $\vec{x}'$ if the oracle returns Yes, otherwise stay at the current point $\vec{x}$. Our data set generation algorithm is given a set of goal states, as well as one initial feasible state to initialize a set of seed states. We use the Sobol sequence~\cite{sobol1967distribution} as the goal states, and the system equilibrium state as the initial feasible state. Our algorithm then iteratively uses these seed and goal sets to efficiently search for other feasible states, as well as their corresponding optimal solution. It then uses these feasible states to update the set of seed states. The full algorithm is defined in Alg.~\ref{alg:space_filling_tree} and Alg.~\ref{alg:line_solve}. Figs.~\ref{fig:DatasetGen1} and \ref{fig:DatasetGen2} provide illustrations of its behavior on Sys. 1. We first generate one large dataset, and then split it into a train set, buffer set, and test set. Due to the sequential nature of our dataset generation algorithm, we introduce this buffer set to ensure that the train and test set do not contain overlapping seed and goal points.

\begin{algorithm}[t]
	\caption{Data Set Generation}
	\label{alg:space_filling_tree}
  \begin{algorithmic}[1]
  \small
	\Require number of goal points $N_{trn}$, $N_{bf}$, $N_{tst}$ for the train set, buffer set, and test set; step size $d > 0$; polyhedron $\mathcal{X}$; primal active-set QP solver $\alpha$
	\Procedure{GenerateData}{$N_{trn}$, $N_{bf}$, $N_{tst}$, $d$, $\mathcal{X}$, $\alpha$}
	  \State Let $\vec{s}_0 = (\vec{x}_{0}, \vec{z}_{0}, \vec{\nu}_{0}, \vec{\lambda}_{0}, \vec{a}_0)$ be an initial seed with \hspace*{2mm} optimal primal, dual, and auxiliary solver variables
	  \State Let $\mathcal{G}_{trn} \cup \mathcal{G}_{bf} \cup \mathcal{G}_{tst} \subset \mathcal{X}$ be a set of $N_{trn} + N_{bf} + N_{tst}$ \hspace*{2mm} points of a Sobol sequence in $\mathcal{X}$
	  \State $\mathcal{S}_{trn}, \mathcal{D}_{trn} = \Call{RandomWalk}{\mathcal{G}_{trn}, \{\vec{s}_0\}, d, \alpha}$
		\State $\mathcal{S}_{bf\phantom{f}}, \mathcal{D}_{bf} = \Call{RandomWalk}{\mathcal{G}_{bf}, \mathcal{S}_{trn}, d, \alpha}$ 
		\State $\mathcal{S}_{tst}, \mathcal{D}_{tst} = \Call{RandomWalk}{\mathcal{G}_{tst}, \mathcal{S}_{bf} \backslash \mathcal{S}_{trn},d, \alpha}$ 
		\State \Return $\mathcal{D}_{trn},\mathcal{D}_{tst}$
	\EndProcedure
	\Procedure{RandomWalk}{$\mathcal{G}$, $\mathcal{S}$, $d$, $\alpha$}
	  \State $\mathcal{D} = \varnothing$
	  \For{$\vec{x}_{g} \in \mathcal{G}$}
	    \State Sample seed tuple $\vec{s} \in \mathcal{S}$
	    \State $\mathcal{D}' = \Call{LineSolve}{\vec{s},\vec{x}_{g}, d, \alpha}$
	    \State $\mathcal{D} = \mathcal{D} \cup \mathcal{D}'$
	    \State Place last seed from $\mathcal{D}'$ in $\mathcal{S}$
	  \EndFor
	  \State \Return $\mathcal{S}, \mathcal{D}$
	\EndProcedure
	\Procedure{LineSolve}{$\vec{s}_0,\vec{x}_g,d,\alpha$}
	  \State Let $\vec{x}$ be the first element of $\vec{s}_0 = (\vec{x}_0, \vec{z}_0, \vec{\nu}_0, \vec{\lambda}_0, \vec{a}_0)$
	  \State $n = \|\vec{x}_{g} - \vec{x}\|$, $\vec{x}_n = (\vec{x}_{g} - \vec{x})/n$ 
	  \For{$i=1 \dots \ceil{n/d}$}
	    \State $\vec{x}_i = \vec{x} + id\vec{x}_n$
	    \State Let $QP(\vec{x}_i)$ be problem~\eqref{eq:decoupled_batch_problem} with parameter $\vec{x}_i$
	    \State $(\textrm{success}, \vec{s}_{i}$) = Hot start $\alpha$ with $\vec{s}_{i-1}$ on $QP(\vec{x}_i)$
	    \If {$\textrm{success}$} $\mathcal{D} = \mathcal{D} \cup \vec{s}_i$
	    \Else{} \textbf{break}
	    \EndIf
	  \EndFor
	  \State \Return{$\mathcal{D}$}
	\EndProcedure
	\end{algorithmic}
\end{algorithm}

\begin{figure}[ht]
\centering
\begin{subfigure}[b]{.475\textwidth}
  \includegraphics[width=\textwidth]{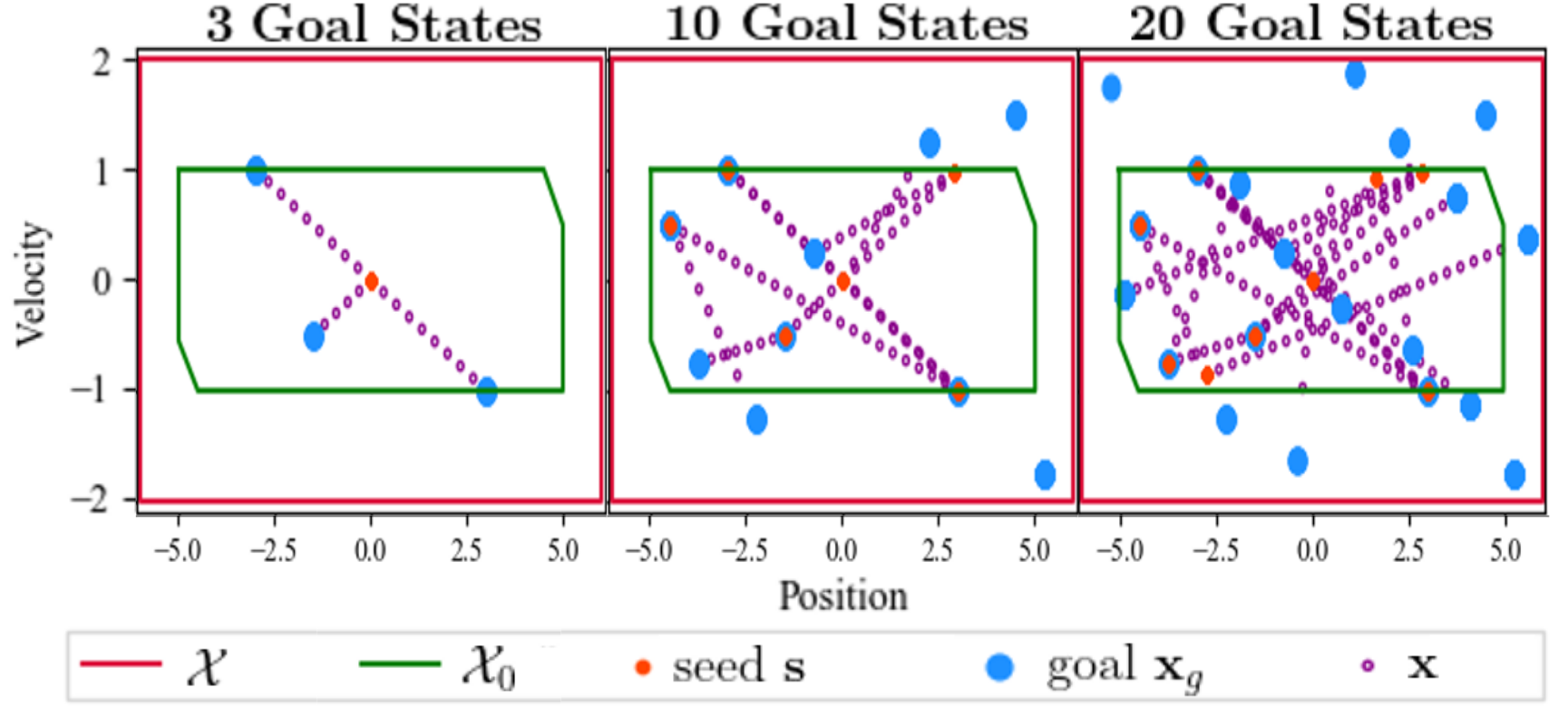}
  \caption{\textbf{Space Filling Generator}: Data generation procedure with varying number of goal states.}
  \label{fig:DatasetGen1}
\end{subfigure} 
\vskip\baselineskip
\begin{subfigure}[b]{.475\textwidth}
  \includegraphics[width=\textwidth]{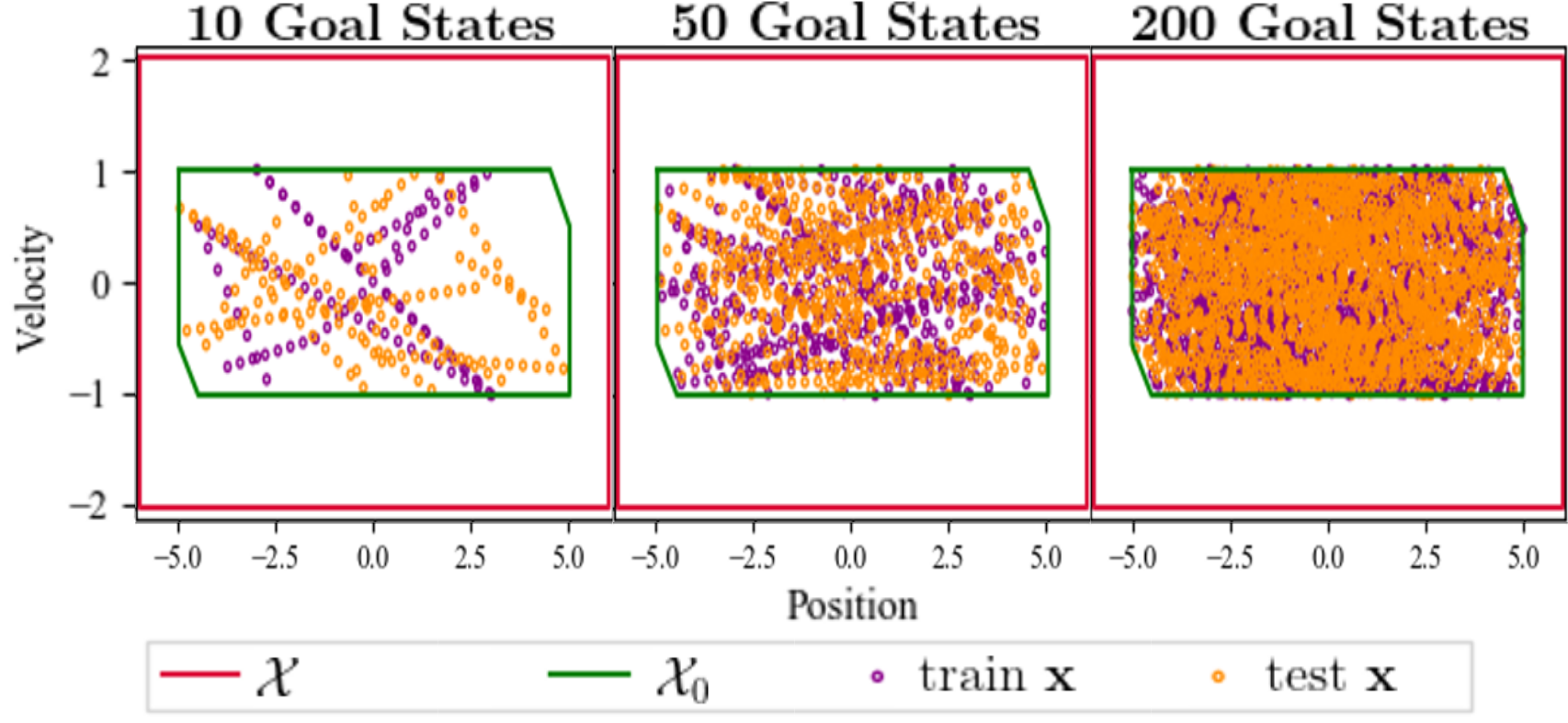}
  \caption{\textbf{Train and Test Distribution}: Illustration of the train and test set generation. Both fill $\mathcal{X}_{0}$ with enough goal states.}
  \label{fig:DatasetGen2}
\end{subfigure}
\caption{Illustration of the data set generation algorithm (Alg~\ref{alg:space_filling_tree}) for a double integrator system defined in Sec.~\ref{sec:system_description}.}
\label{fig:DatasetGen}
\end{figure}

%% file: tex/Examples.tex
\section{Evaluation}
%\section{Numerical Examples}
\label{sec:results}

%
%\begin{figure*}[t!]
%	\begin{center} 
%		\setlength\extrarowheight{5pt}
%		%\resizebox{\textwidth}{!}
%		{\begin{tabular}{||c | c | c | c | c | c | c||} 
%				\hline
%				\thead{System} & \makecell{\textbf{Train} \\ Dataset Size \\ (Sobol Size)} & \makecell{\textbf{Test Seed} \\ Dataset Size \\ (Sobol Size)} & \makecell{\textbf{Test} \\ Dataset Size \\ (Sobol Size)}& \makecell{\textbf{Dataset} \\ \textbf{Generation} \\ \textbf{Time}} & \makecell{\textbf{Training} \\ Time \\ (Epochs)} & \makecell{\textbf{Testing} \\ \textbf{Time}} \\ [0.6ex] 
%				\hline\hline
%				Double Integrator & \makecell{17,873 \\ (2,000)} & \makecell{3,754 \\ (400)} & \makecell{3,683 \\ (400)} & 2.6s & 120s & 3.98s\\
%				\hline\hline
%				Oscillating Masses & 0 & 0 & 0 & 0 & 0 & 0 \\
%				\hline\hline
%				\makecell{High-Dimensional \\ Oscillating Masses} & 
%				\makecell{$1365k$ \\ ($200k$)} & \makecell{$273k$ \\ ($40k$)} & \makecell{$273k$ \\ ($40k$)} & $3.5$ hrs & \makecell{2.0 hrs \\ (50)} & 24 hrs  \\
%				\hline\hline
%			\end{tabular}}
%		\end{center}
%		\caption{Dataset and Training Statistics:  }
%		\label{fig:training_stats}
%\end{figure*}

We demonstrate the proposed approach on four systems of increasing complexity. The experiments highlight the challenges that arise when scaling to high-dimensional systems. We use MPT3 \cite{MPT3} to compute the terminal cost and constraints in \eqref{eq:time_invariant_finite_horizon_problem}, Tensorflow \cite{abadi2016tensorflow} to train and evaluate the neural network planner $\tilde{\pi}(\vec{x}|\vec{\theta})$, and the primal active-set QP solver SQOPT \cite{sqopt77} to implement the Phase I and Phase II iterations in Alg. \ref{alg:explicit_implicit_planner}. The neural network models were trained using dual Xeon E5-2683 v4 CPUs with $32$ cores for data generation and two NVIDIA Titan X Pascal GPUs for training. The online inference speed evaluation was performed on an i7-7700K CPU with $4$ cores and NVIDIA GeForce GTX 1080 Ti GPU.

% We use a variety of software packages to implement our approach. 
%We provide numerical examples of our approach on four systems. The purpose of these examples is to demonstrate that our proposed approach has the ability to scale to large problems, as well as highlight the challenges that arise when scaling to a high dimensional system. We use Tensorflow~\cite{abadi2016tensorflow} to train and evaluate the neural networks, SQOPT~\cite{sqopt77} as our primal active set solver, and MPT3~\cite{MPT3} to compute the terminal cost and constraints. The network models were trained using dual Xeon E5-2683 v4 CPUs with $32$ cores for the data set generation, and two NVIDIA Titan X Pascal GPUs for the training. The evaluation of online inference speed was performed on an i7-7700K CPU with $4$ cores, and an NVIDIA GeForce GTX 1080 Ti GPU.

% in \eqref{eq:least_squares_Lagrangian}
% Each hidden layer is followed by the ReLU activation function. 
The networks were trained with the Lagrangian loss function $\ell(\vec{\theta})$ using the Adam optimizer \cite{kingma2014adam} for 100 epochs. The network architectures and training parameters are listed in Fig.~\ref{fig:nn_structure}. The network depths were chosen using techniques described in \cite{Chen2018ApproximatingEM} and the widths were chosen by trial and error. Other methods to choose network architectures are described in \cite{lucia2018deep}. Given the large sizes of the train data sets (see Fig.~\ref{fig:data_set_overview}), 100 epochs is sufficient to reach convergence. Depending on the system, we take $2.5-10 \%$ samples from the train set as a held-out validation set. The validation set is only used to evaluate the training loss to avoid over- or under-fitting and is not used to update the neural network weights.

\begingroup
\renewcommand{\arraystretch}{1.1}
\begin{figure}[t!]
    \begin{center} 
\centering
    \begin{tabular}{| c | c | c |}
        \hline
        System & Layer Widths & Num. Parameters \\
        \hline
        Sys. 1 & \textbf{2}, 32, 32, \textbf{30} & 2,142\\
        \hline
        Sys. 2 & \textbf{12}, 32, 32, \textbf{300} & 89,900\\
        \hline 
        Sys. 3  & \makecell{\textbf{12}, 32, 64, 128, \\ 256, \textbf{450}} & 159,522\\
        \hline
        Sys. 4 & \makecell{\textbf{36}, 128, 128, 256, \\ 256, 512, 512, \textbf{2250}} & 1,668,554\\
        \hline
    \end{tabular}
    \end{center}
        \caption{Neural network architectures. Each layer is fully connected and followed by the ReLU nonlinearity. Bolded are the input and output sizes of the network.}
        \vspace{2mm}
      \label{fig:nn_structure}
\end{figure}
\endgroup

\begingroup
\setlength{\tabcolsep}{4pt}
\begin{figure}
\begin{center}
\begin{tabular}{| c | c | c | c | c | c | c | c | c | c |}
\hline
System & $n$ & $m$ & $N$ & $c_x$ & $c_f$ & $c_u$ & $d_p$ & $d_{in}$ & $d_{eq}$ \\
\hline
Sys. 1 & 2 & 1 & 10 & 4 & 6 & 2 & 30 & 66 & 20 \\
\hline 
Sys. 2 & 12 & 3 & 20 & 24 & 246 & 6 & 300 & 846 & 240 \\
\hline 
Sys. 3 & 12 & 3 & 30 & 24 & 104 & 6 & 450 & 1004 & 360 \\
\hline 
Sys. 4 & 36 & 9 & 50 & 72 & 408 & 18 & 2250 & 4908 & 1800 \\
\hline 
\end{tabular}
\end{center}
\caption{Problem specifications: state dimension ($n$), control dimension ($m$), time horizon ($N$), state constraints ($c_x$), terminal state constraints ($c_f$), control constraints ($c_u$), primal variable dimension ($d_p$), dual variable dimension for inequality constraints ($d_{in}$) and equality constraints ($d_{eq}$).}
\label{fig:problem_specifications}
\end{figure}
\endgroup

%The first system is the double integrator system that has been used as an illustrative example, and it has a system size of $n=2$, $m=1$, and $N=10$. The second system is a quadrotor with a system size of $n=12$, $m=3$, and $N=20$. The third system is an oscillating masses system with size of $n=12$, $m=3$, and $N=30$. The fourth system is a high dimensional version of the oscillating masses system and has a size of $n=36$, $m=9$, and $N=50$.

%We first introduce each system in Section~\ref{sec:system_description}. We then present quantitative results and details of the training of the neural network in Section~\ref{sec:training_statistics}. We next evaluate our RHC on a held out test set on the matrics of required number of active set iterations before termination, and the resulting sub-optimality of the RHC in Section~\ref{sec:test_metrics}. We then present some statistics that highlight and illustrate the computational challenges that occur when scaling to large systems in Section~\ref{sec:scaling_metrics}

\subsection{System Descriptions}
\label{sec:system_description}
The dimensions of the four problem settings are summarized in Fig.~\ref{fig:problem_specifications}. We choose the terminal region $\mathcal{X}_{f}$ to be $\mathcal{O}_{\infty}^{LQR}$, and the terminal cost $\vec{x}_{N}^{\top}\mat{P}_{\infty}\vec{x}_{N}$, defined in \eqref{eq:dare}.

\begingroup
\renewcommand{\arraystretch}{1.2}
\begin{figure*}
\begin{tabularx}{\textwidth}{ |c| *{6}{Y|} }
\cline{2-7}
	\multicolumn{1}{c|}{}
 & \multicolumn{3}{c|}{\textit{Data Set Sizes}}
 & \multicolumn{3}{c|}{\textit{Timing Statistics}} \\
\cline{2-7}
   \multicolumn{1}{c|}{} 
 & \multicolumn{1}{c|}{\textbf{Train}}
 & \multicolumn{1}{c|}{\textbf{Buffer}}
 & \multicolumn{1}{c|}{\textbf{Test}}   
 & \multicolumn{1}{c|}{\textbf{Data Gen.}}
 & \multicolumn{1}{c|}{\textbf{Network Train}}
 & \multicolumn{1}{c|}{\textbf{Test}}\\
 \hline
 \cline{1-7}
    \multicolumn{1}{|c|}{Sys.}
 & \multicolumn{3}{c|}{\# \textbf{Total} (\textit{Seed}) \textit{examples in thousands}}
 & \multicolumn{1}{c|}{\textit{total}}
 & \multicolumn{1}{c|}{\textit{total}}
 & \multicolumn{1}{c|}{\textit{per example}} \\
 \hline
 1  & \textbf{\ \ 15} (2) &  \textbf{\ \ \ \ \ 4} (0.4) &  \textbf{\ \ \ \ \ 4} (0.4) & ~1.5 secs & 100 secs & ~~1.0 ms\\
 2  &  \textbf{\ \ \ 863} (200) & \textbf{173} (40) &  \textbf{173} (40) &  ~~3.0 mins & 1.4 hrs & ~~3.1 ms\\
 3  &  \textbf{\ 174} (20) & \textbf{35} (4) & \textbf{35} (4) &  ~~1.0 mins & 0.4 hrs & ~~3.3 ms\\
 4  &  \textbf{1,754} (400) & \textbf{175} (40) & \textbf{175} (40) & 5.2 hrs & ~20 hrs & 38.8 ms\\
\hline
\end{tabularx}
        \caption{Overview of the data set sizes and offline data generation, train, and test time statistics.}
        \label{fig:data_set_overview}
\end{figure*}
\endgroup

\begin{system}[Double Integrator]
\label{sys:double_integrator}
A mass under force input can be modeled as a double integrator system,
\begin{equation*}
\vec{x} \in \mathbb{R}^{2}, \;\;\; \vec{u} \in \mathbb{R}, \;\;\; \mat{A} = \begin{bmatrix}
1 & 1 \\ 0 & 1
\end{bmatrix}, \;\;\; \mat{B} = \begin{bmatrix}
0.5 \\ 0.1
\end{bmatrix}.
\end{equation*}
We consider a problem with cost matrices $\mat{Q}=\mat{I}_{2}$, $\mat{R}=1$, state constraints $\mat{A}_{x} = [\mat{I} \;{-\mat{I}}]^{\top}$, $\vec{b}_{x} = [5 ~ 1 ~ 5 ~ 1]^{\top}$, and input constraints $|u_k|\leq 2$, for $k=1,\ldots,N$. 
\end{system}

\begin{system}[Quadrotor]
\label{sys:quadrotor}
The differential flatness of quadrotor robot dynamics enables trajectory generation in the space of flat outputs (position, yaw and their derivatives) that are dynamically feasible for an underactuated quadrotor \cite{liu2017search}. This leads to a time-invariant continuous-time system $\dot{\vec{x}} = \mat{A}_{c}\vec{x} + \mat{B}_{c}\vec{u}$, where $\mat{A}_c = \mat{L}_4 \otimes \mat{I}_3$ and $\mat{B}_c = \vec{e}_4 \otimes \mat{I}_3$. The $3 \times 3$ sub-matrices correspond to position, velocity, acceleration, and jerk. We discretize the system using Euler discretization with time step $0.1$, and consider constraints $\mat{A}_{x} = [\mat{I}_{12}, -\mat{I}_{12}]^{\top}$, $\vec{b}_{x} = [10\cdot \vec{1}_{3}, 5\cdot \vec{1}_{3}, 3 \cdot \vec{1}_{3}, 1 \cdot \vec{1}_{3}, 10 \cdot \vec{1}_{3}, 5\cdot \vec{1}_{3}, 3 \cdot \vec{1}_{3}, 1 \cdot \vec{1}_{3}]^{\top}$, $\mat{A}_{u} = [\mat{I}_{3}, -\mat{I}_{3}]^{\top}$, $\vec{b}_{u} = [\mat{1}_{3}, \mat{1}_{3}]^{\top}$ and cost matrices $\mat{Q} = \mat{I}_{12}$, $\mat{R} = \mat{I}_3$. 
\end{system}

%this continuous time 
% mellinger2012trajectory,

\begin{system}[Oscillating Masses]
\label{sys:oscillating_masses}
Introduced in \cite{wang2010fast}, the oscillating masses is a linear system that can be scaled to large dimensions by increasing the number of masses and springs in the system. We use $6$ masses with a mass of $1$, and $3$ springs with a spring constant $c = 1$ and damping constant $d=0.1$. Let $a = -2c$ and $b=-2$. The system is defined as $\dot{\vec{x}} = \mat{A}_{c}\vec{x} + \mat{B}_{c}\vec{u}$, where 
\begin{align}
\mat{A}_c &= \begin{bmatrix} \mat{0}_6 & \mat{I}_6 \\ a \mat{I}_6 + c\mat{L}_6+c\mat{L}^\top_6 & b \mat{I}_6 + d\mat{L}_6+d\mat{L}^\top_6 \end{bmatrix} \quad \mat{B}_{c} = \begin{bmatrix} \mat{0} \\ \mat{F} \end{bmatrix} \notag\\
%\mat{F} &= \left[
%\begin{array}{ccc}
%1 & 0 & 0  \\
%-1 & 0 & 0  \\
%0 & 1 & 0  \\
%0 & 0 & 1  \\ 
%0 & -1 & 0  \\
%0 & 0 & 1  \\  
%\end{array} \right],
\mat{F} &= \begin{bmatrix}
\vec{e}_1 & -\vec{e}_1 & \vec{e}_2 & \vec{e}_3 & -\vec{e}_2 & \vec{e}_3 
\end{bmatrix}^\top \in \mathbb{R}^{6 \times 3}. \label{eq:oscillating_A}
%\quad \mat{B}_{c} =
%\left[
%\begin{array}{c}
%\mat{0}  \\ \hdashline
%\mat{F} \\ 
%\end{array} \right]. \notag
\end{align}
We discretize this system using first-order hold and time step of $0.5$. The cost terms are $\mat{Q}=\mat{I}_{12}$, $\mat{R}=\mat{I}_{3}$, subject to state constraints $\mat{A}_{x} = [\mat{I}, -\mat{I}]^{\top}$, $\vec{b}_{x} = 4 \cdot \vec{1}$, and input constraints $\mat{A}_{u} = [\mat{I}, -\mat{I}]^{\top}$, $\vec{b}_{u} = 0.5  \cdot \vec{1}$. 
\end{system}

\begin{system}[High Dimensional Oscillating Masses]
\label{sys:high_dim_oscillating_masses}
The high dimensional oscillating masses system is a scaled-up version of Sys.~\ref{sys:oscillating_masses}. The matrix $\mat{A}_{c}$ is the extension of \eqref{eq:oscillating_A} to 36 dimensions obtained by replacing the $6$-dimensional submatrices with $18$-dimensional counterparts. The input matrix is $\mat{B}_{c} = [ \mat{0} ~|~ \mat{I}_{3} \otimes \mat{F}^{\top} ] ^{\top}$. The cost and constraint terms are chosen similarly as for Sys.~\ref{sys:oscillating_masses}. 
%Note that the maximal positively invariant set of the LQR controller, $\mathcal{O}_{\infty}^{LQR}$, is simple to calculate even for large-dimensional systems using off-the-shelf toolboxes~\cite{MPT3}. 
\end{system}

\begingroup
\renewcommand{\arraystretch}{1.2}
\begin{figure*}
\begin{tabularx}{\textwidth}{ |c| *{6}{Y|} }
\cline{2-7}
   \multicolumn{1}{c|}{} 
 & \multicolumn{3}{c|}{Neural Network Warm Start}  
 & \multicolumn{3}{c|}{Cold Start}\\
\cline{2-7}
  \multicolumn{1}{c|}{}  & \textit{p.f.} & \textcolor{magenta}{\textit{p.f. + sub.}} & \textit{optimal} & \textit{p.f.} & \textit{p.f. + sub.} & \textcolor{olive}{\textit{optimal}}\\
\hline
 \multicolumn{1}{|c|}{Sys.} 
 &\multicolumn{6}{c|}{\makecell{\textbf{Absolute Iterations} Average (\textit{Worst})}}\\
 \hline
 1  & 1.0 (5) &  \textcolor{magenta}{7.1} (\textcolor{magenta}{15}) &  12.3 (21) & 1.1 (2) & 9.5 (12) & \textcolor{olive}{11.4} (\textcolor{olive}{22})\\
 2  &  8.2 (44) & \textcolor{magenta}{17.9} (\textcolor{magenta}{70}) &  57.9 (110) &  32.5 (97) & 62.1 (124) & \textcolor{olive}{98.9} (\textcolor{olive}{144})\\
 3  &  5.3 (123) & \textcolor{magenta}{7.2} (\textcolor{magenta}{133}) & 83.6 (145) &  224.6 (408) & 255.1 (448) & \textcolor{olive}{327.1} (\textcolor{olive}{498}) \\
 4  &  17.4 (291) & \textcolor{magenta}{27.8} (\textcolor{magenta}{391}) & 400.4 (1115) & 1074.4 (2089) & 1262.1 (2409) & \textcolor{olive}{1611.7} (\textcolor{olive}{2559}) \\
\hline
 \multicolumn{1}{|c|}{} 
 &\multicolumn{6}{c|}{\makecell{\textbf{Open Loop Sub. (\%) $\sigma_{ol}$} Average (\textit{Worst})} }\\
 \hline
 1  & \small{1277.6 (41277.2)}  & \textcolor{magenta}{9.8} (\textcolor{magenta}{266.3}) & --- & 201.3 (583.4) & 13.0 (277.8) & ---\\
 2  &  1.2 (29.3)  & \textcolor{magenta}{0.3} (\textcolor{magenta}{7.0}) &  --- &  19.7 (161.8) & 0.8 (11.1) & ---\\
 3  &  3.1 (16.9) & \textcolor{magenta}{2.6} (\textcolor{magenta}{10.1}) & --- &  98.3 (423.4) & 6.4 (33.0) & ---\\
 4  &  4.9 (26.4) & \textcolor{magenta}{3.5} (\textcolor{magenta}{13.8}) & --- & 192.9 (686.6) & 6.3 (22.5) & ---\\
\hline
\end{tabularx}
        \caption{Open-loop metrics: average and worst case number of iterations and suboptimality. Our method is displayed in \textcolor{magenta}{magenta}, and the baseline method is displayed in \textcolor{olive}{olive}. The evaluations are performed on independent states in the test set, and the baselines methods can only be initialized from a Cold start. On Sys. 4, \textcolor{magenta}{our method} achieves a $58\times$ reduction in required iterations compared to the \textcolor{olive}{baseline}. Our methods achieve better suboptimality metric $\sigma_{ol}$ compared to Cold starts due to the higher-quality NN initialization.}
        \label{fig:test_stats}
\end{figure*}
\endgroup

\subsection{Data Generation}

As discussed in Sec.~\ref{sec:space_filling}, $\mathcal{X}_{0}$ cannot be sampled directly. Without an efficient data set generation strategy, obtaining large data sets is practically impossible. For example, a naive rejection-sampling approach yields feasible samples $98.6\%$ of the time for Sys.~\ref{sys:double_integrator}, $16.7\%$ for Sys.~\ref{sys:quadrotor}, $1.1\%$ for Sys.~\ref{sys:oscillating_masses}, and $0.4\%$ for Sys.~\ref{sys:high_dim_oscillating_masses}. Generating a data set of $1$ million feasible states for Sys.~\ref{sys:high_dim_oscillating_masses} using rejection sampling will take $2.5$ years. In contrast, the geometric random walk approach in Alg.~\ref{alg:space_filling_tree} generated a data set of $1.7$ million feasible states in $5.2$ hours. Fig.~\ref{fig:data_set_overview} shows the data set sizes used for the different systems and the computation time for data generation using Alg.~\ref{alg:space_filling_tree}.

\subsection{Open-loop Metrics}
\label{sec:test_metrics}

\begingroup
\renewcommand{\arraystretch}{1.2}
\begin{figure*}
\begin{tabularx}{\textwidth}{ |c| *{6}{Y|} }
\cline{1-7}
   \multicolumn{1}{|c|}{Sys. 4} 
 & \multicolumn{4}{c|}{\textbf{Online Inference Time (ms)}}
  & \multicolumn{2}{c|}{\textbf{Trajectory}}  \\
 \cline{1-5}
   \multicolumn{1}{|c|}{\textbf{ Method}} 
 & \multicolumn{1}{c|}{\textbf{Network}}
 & \multicolumn{1}{c|}{Initial \textbf{$\vec{x}(0)$}} 
 & \multicolumn{1}{c|}{Remaining \textbf{$\vec{x}(t)$}}
 & \multicolumn{1}{c|}{\textbf{Total}}
 & \multicolumn{1}{c|}{\textbf{Sub.~(\textit{\%}) $\sigma_{cl}$}} \\
 
 \cline{1-7}
\multicolumn{7}{|c|}{Average (\textit{Worst})} \\
 
\cline{1-7}
    \textcolor{magenta}{NN \textit{p.f. + sub.}} & \textcolor{magenta}{~~11 (\textit{65})} & ~~\textcolor{RedOrange}{29} (\textcolor{RedOrange}{\textit{168}}) & \textcolor{RedOrange}{~~15} (\textcolor{RedOrange}{\textit{92}}) & \textcolor{magenta}{~~664 (\textit{1678})} & \textcolor{magenta}{8.8 (\textit{15.5})} \\
    NN ($\eta < 0.1$) & 12 (\textit{63}) & \textcolor{RedOrange}{114} (\textcolor{RedOrange}{\textit{168}}) & \textcolor{RedOrange}{88} (\textcolor{RedOrange}{\textit{216}}) & 2185 (\textit{5380}) & 0.0 (\textit{0.0}) \\
    NN \textit{optimal} & 12 (\textit{63}) & \textcolor{RedOrange}{200} (\textcolor{RedOrange}{\textit{221}}) & \textcolor{RedOrange}{193} (\textcolor{RedOrange}{\textit{259}}) & 4420 (\textit{8355}) & --- \\
 \cline{1-7}
    Hot \textit{p.f. + sub.} & --- & \textcolor{blue}{447} (\textcolor{blue}{\textit{576}}) & ~~\textcolor{red}{41} (\textcolor{red}{\textit{186}}) & 1311 (\textit{2988}) & 0.6 (\textit{1.5}) \\
    Hot ($\eta < 0.1$) & --- & \textcolor{blue}{543} (\textcolor{blue}{\textit{644}}) & \textcolor{red}{69} (\textcolor{red}{\textit{288}}) & 1962 (\textit{3873}) & 0.0 (\textit{0.0}) \\     
    Hot \textit{optimal} & --- & \textcolor{blue}{593} (\textcolor{blue}{\textit{676}}) & \textcolor{red}{89} (\textcolor{red}{\textit{474}}) & 2409 (\textit{4259}) & --- \\     
 \cline{1-7}
    CVXPY (\textit{ECOS}) & --- & ~~\textcolor{blue}{276} (\textcolor{blue}{\textit{1541}}) & \textcolor{blue}{216} (\textcolor{blue}{\textit{386}}) & 5114 (\textit{10141}) & --- \\    
    CVXPY (\textit{Gurobi}) & --- & \textcolor{blue}{147} (\textcolor{blue}{\textit{390}}) & \textcolor{blue}{151} (\textcolor{blue}{\textit{192}}) & 3525 (\textit{6190}) & --- \\   
    CVXPY (\textit{Mosek}) & --- & ~~\textcolor{blue}{182} (\textcolor{blue}{\textit{1437}}) & \textcolor{blue}{161} (\textcolor{blue}{\textit{225}}) & 3789 (\textit{7032}) & --- \\    
 \cline{1-7}
\end{tabularx}
\caption{%
Closed-loop metrics (Sys. 4): Average and worst case time required to implement the \textit{RHC} for \textcolor{RedOrange}{Warm starts}, \textcolor{Red}{Hot starts}, and \textcolor{Blue}{Cold starts}. The closed-loop trajectory suboptimality metric $\sigma_{cl}$ (defined after \eqref{eq:closed-loop-cost}) is different from the open-loop suboptimality metric $\sigma_{ol}$. The former is a cumulative measure computed along the entire closed-loop trajectory, while the latter is an instantanous measure at a given state $\vec{x}$. The CVXPY benchmarks are state-of-the-art QP solvers that provide a baseline for the online inference speed of the active set methods. The results show that relative to Hot starts, the NN Warm starts quickly reach the conditions specified in Thm.~\ref{thm:approximate_RHC_stability} (represented by \textit{p.f. + sub.}).
%approximate thresholds (represented by \textit{p.f. + sub.}). 
The NN Warm starts then switch from faster to slower relative to Hot starts captured at the fixed duality gap threshold $\sigma(\vec{z}|\vec{x}) \leq \eta(\vec{z},\vec{\nu},\vec{\lambda}|\vec{x}) < 0.1$. The NN methods proceed slowly to the precise thresholds needed for optimality. Practically, these latter iterations are unnecessary, because when NN and Hot are terminated with $\eta < 0.1$, the resulting primal variables are essentially optimal, and the closed-loop trajectory matches the optimal one with $\sigma_{cl} = 0.0$. While NN \textit{p.f. + sub.} does have a higher $\sigma_{cl}$ compared to Hot \textit{p.f. + sub.}, this difference indicates the effectiveness of pairing early termination criteria with NN warm starts, as Hot starts need to bear additional computation time to reach near optimal solutions before obtaining the required termination certificates.%
}
% Trajectory suboptimality is different than open loop suboptimality defined in Eqn.~\eqref{eq:suboptimality_level}, as the trajectory suboptimality is computed across the entire trajectory, whereas Eqn.~\eqref{eq:suboptimality_level} is an instantaneous planning cost computed at each encountered state $\vec{x}(t)$. 
%\caption{(Sys. 4) Average time required to compute the \textit{RHC} for our method compared to benchmark methods. We separate the time statistics into neural network inference time, the solution time for the initial state $\vec{x}(0)$, the solution times for the remainder states, and the total solution time across the trajectory. We also report the suboptimality performance as a $\%$ of the optimal cost.}
    \label{fig:avg_time}
\end{figure*}
\endgroup

% The baselines include state-of-the-art interior-point QP solvers, used to solve \eqref{} online as well as an active set method initialized using a 

On the test set, we compare our method (magenta) in Fig.~\ref{fig:test_stats} to various baseline methods. The baselines are the $6$ combinations between $2$ initialization methods -- NN Warm start and Cold start -- and $3$ termination criteria -- terminating after primal feasibility (\emph{p.f.}), primal feasibility and suboptimality (\emph{p.f. + sub.}), and optimality. \emph{Sub.} is achieved at the time when the duality gap $\eta(\vec{z},\vec{\nu},\vec{\lambda}|\vec{x})$ defined in \eqref{eq:duality_bound_suboptimality} is less than $\vec{x}^{\top}\mat{Q}\vec{x}$ according to Thm. \ref{thm:approximate_RHC_stability}. The main benchmark is Cold start and solving to optimality (olive). Fig.~\ref{fig:test_stats} displays the results on the four systems according to various initialization and termination criteria. We evaluate speedup by reporting the number of required iterations and evaluate performance by reporting the suboptimality levels. For Sys. 4, compared to the benchmark Cold \textit{optimal} control law, our method (NN \textit{p.f. + sub.}) reduces required iterations by $58\times$. The results demonstrate that early termination is especially synergistic when paired with NN Warm starts, as it reduces the required iterations by $14.4\times$ versus only $1.3\times$ when paired with Cold starts.

%In addition, our results demonstrate that our ability to terminate the online solver early also drastically cuts down on the required number of iterations. Our method requires on average $400.4$ iterations and the Cold start method requires $1611.7$ iterations to reach optimality, demonstrating the speed up benefits of terminating the solver once the primal feasibility and suboptimality criteria have been reached.

Our method additionally performs well in terms of the suboptimality level. We define \emph{open loop sub. (\%)} as $\sigma_{ol} := \frac{J(\vec{z}|\vec{x}) - J(\vec{z}^{*}|\vec{x})}{J(\vec{z}^{*}|\vec{x})}$. On Sys.~\ref{sys:high_dim_oscillating_masses}, it has an average $\sigma_{ol}$ of $4.9\%$ when only the \textit{p.f.} criterion is met, and $3.5\%$ when both \textit{p.f. + sub.} criteria are met. This is in contrast to the Cold start method, where $\sigma_{ol}$ is high ($192.9\%$) when only the \textit{p.f.} criterion is met, and obtaining both \textit{p.f. + sub.} criteria is necessary to obtain a reasonable suboptimality ($6.3\%$). This discrepancy demonstrates that the neural network $\tilde{\pi}$ alone is obtaining a good guess of the optimal solution, and the online iterations serve mostly to satisfy the \textit{p.f.} criterion. For the low dimensional Sys. 1, $\sigma_{ol}$ of both the NN Warm start and Cold start methods for the \textit{p.f.} criteria is high because the optimal cost is very small ($\sim0$). 

%In addition to measuring the number of iterations, Fig.~\ref{fig:test_stats} also measures 
%the quality of each method according to the suboptimality level compared to 
%the baseline method.
%For the high dimensional systems, when initializing with the neural network, 
%the suboptimality level is low even when only the primal feasibility condition is satisfied. 
%This is in contrast to the Cold start methods, where only achieving primal feasibility 
%results in a high suboptimality level, and the suboptimality criteria must be met 
%in order to obtain good performance. 
%This discrepancy demonstrates that the neural network alone is obtaining a good guess 
%in terms of the objective function, and the iterations serve mostly as slight modifications 
%to obtain satisfaction of the constraints and suboptimality criteria. 

\subsection{Closed-loop Metrics}
\label{sec:high_dim_oscillating_masses_results}

We compare the closed-loop performance of our method to various benchmarks on Sys.~\ref{sys:high_dim_oscillating_masses} in Fig.~\ref{fig:avg_time}. While a Cold start is unavoidable at the initial state $\vec{x}(0)$, we benchmark against Hot start initialization techniques \cite{ferreau2014qpoases} that exploit solutions of previously solved states for the subsequent states $\vec{x}(t)$. The first benchmark is the SQOPT solver \cite{sqopt77} using Hot starts. The other benchmarks include state-of-the-art interior-point QP solvers, ECOS, Gurobi, and Mosek, from CVXPY \cite{diamond2016cvxpy}. For the SQOPT methods, we evaluate $3$ termination criteria: ours (\textit{p.f. + sub.}), a fixed duality gap threshold $\eta(\vec{z},\vec{\nu},\vec{\lambda}|\vec{x}) < 0.1$, and optimal. The CVXPY methods cannot be terminated early, and warm/hot starts are less effective.

% Starting at these initial states $\vec{x}(0)$, a
We evaluate on trajectories generated from $128$ initial states $\vec{x}(0)$ randomly sampled from the test set. At each subsequent state $\vec{x}(t)$, we apply the first control $\vec{u}_0$ obtained from each method. We repeatedly execute this closed-loop controller until the system reaches a state $\vec{x}(t) \in \mathcal{X}_{f}$, at which point each method switches to an optimal LQR controller \cite[Ch. 8]{borrelli2017predictive}. Inference speed is evaluated on states $\vec{x}(t) \notin \mathcal{X}_{f}$, and these states will vary across each method. Optimality performance is evaluated on the entire infinite-horizon trajectory:
\begin{equation} \label{eq:closed-loop-cost}
J_{cl} = \sum_{t=0}^{\infty}\vec{x}^{\top}(t)\mat{Q}\vec{x}(t) + \vec{u}^{\top}(t)\mat{R}\vec{u}(t),
\end{equation}
where the \emph{trajectory suboptimality} is $\sigma_{cl} = \frac{J_{cl} - J^*_{cl}}{J^*_{cl}}$.

For methods that terminate early, rather than performing the suboptimality check at every iteration, it is more efficient to perform the check periodically. In our implementation with SQOPT, API restrictions require the suboptimality checks to be performed externally to the solver, and as a result, we periodically terminate the solver to check the intermediate iterates. We utilize the SQOPT API optimality tolerance parameter, which evaluates the size of the reduced gradients \cite[pg. 39]{sqopt77} to prematurely terminate the solver and then perform the suboptimality check with the procedure described in Sec.~\ref{sec:certificates_feasibility_suboptimality} at each resulting iterate. If we are unable to obtain a \textit{sub.} certificate, we reduce the optimality tolerance parameter by $2$ and repeat from the current SQOPT solver state. For consistency, we apply the same termination strategy across all evaluated methods.

The closed-loop experiments with Hot starts (Fig.~\ref{fig:avg_time}) exhibit similar patterns to the open-loop experiments with Cold starts (Fig.~\ref{fig:test_stats}). On average, across the entire trajectories, our method achieves a $2\times$ speedup compared to Hot \textit{p.f. + sub.}, and a $3.6\times$ speedup compared to Hot \textit{optimal}. Similar to the open loop experiments, the closed loop experiments also demonstrate that early termination is especially synergistic with NN Warm starts compared to Hot starts. Versus optimal, early termination reduces inference time by $12.9\times$ for NN compared to $2.2\times$ for Hot. Further analysis suggests that early termination works well with NN because NN initialized methods reach approximate thresholds quickly, and then slowly reach the precise thresholds required for optimality. While, NN \textit{p.f. + sub.} is $2\times$ faster than Hot \textit{p.f. + sub.}, it becomes $1.8\times$ slower when solved to optimality. To demonstrate this switch, Fig.~\ref{fig:avg_time} evaluates each initialization method with a fixed duality gap threshold $\eta < 0.1$ which shows that the NN Warm start inference time is only $1.1\times$ slower than Hot. Importantly, when terminating with $\eta < 0.1$, the trajectories have effectively $0$ suboptimality, demonstrating that improving from an approximate threshold to a precise optimality threshold has little benefit. 

The $\sigma_{cl}$ metrics in Fig.~\ref{fig:avg_time} further highlight synergies between NN Warm starts and early termination. NN \textit{p.f. + sub.} has $\sigma_{cl}$ of $8.8\%$ compared to $0.6\%$ for Hot \textit{p.f. + sub.} While this may initially seem like a drawback for the NN methods, it actually demonstrates that early termination is not as effective in reducing online inference time for Hot starts, as satisfying the termination criteria requires near optimal solutions. The motivation of these methods is fast online inference speed with guarantees of feasibility and stability, not in achieving optimality. 

%% file: tex/Conclusion.tex
\section{Conclusion}

%We presented a hybrid explicit-implicit MPC procedure that combines an offline trained neural network with an online primal active set solver. The neural network component allows our approach to scale to high dimensional problems, while the primal active set solver provides corrective steps to meet the conditions necessary to provide guarantees on recursive feasibilty and asymptotic stability. Our numerical results demonstrate that our approach scales to large problems including one with $36$ states, $9$ inputs, and time horizon of $50$.

We presented a hybrid explicit-implicit MPC procedure that combines an offline trained neural network with an online primal active set solver. We proposed a primal-dual loss function based on the Lagrangian to train the neural network. Using the primal variable predictions, we derived an algorithm to provide certificates of primal feasibility and suboptimality, the criteria necessary to guarantee recursive feasibility and asymptotic stability. Finally, we demonstrated how warm start and early termination can combine the primal active set solver with the neural network to accelerate inference times. The key challenge of function approximation in high dimensional mp-QPs is choosing good training points in a convex set defined by a membership oracle. We introduced a geometric random walk algorithm that can efficiently generate a data set for large problems. Our results indicate the importance of addressing this data generation problem to obtain a scalable solution. The combination of these ideas yields a RHC with guarantees on recursive feasibility and asymptotic stability, while achieving a $2\times$ speedup versus the best benchmark method  on a system with thousands of optimization variables.

%, as a naive method would take $23.8$ years to generate a large data set, compared to our method which only takes $3.5$ hours on the most challenging system.

%There are a variety of directions for future research. One direction is explicitly minimizing the number of iterations required by the active set method during the neural network training process. Another direction is combining the ideas of \textit{geometric random walks} and \textit{QMC} more formally than the algorithm presented in this work. A third direction is using ideas in \textit{Analytical Learning Theory} and \textit{QMC} to create good data sets that result in useful bounds on the generalization of the neural network. The final direction is applying the presented ideas to more complicated problems and systems, such as nonlinear MPC or hybrid mixed-integer MPC.

%% file: autosam.bbl
\begin{thebibliography}{10}

\bibitem{abadi2016tensorflow}
M.~Abadi, P.~Barham, J.~Chen, Z.~Chen, A.~Davis, J.~Dean, M.~Devin,
  S.~Ghemawat, G.~Irving, M.~Isard, et~al.
\newblock {Tensorflow: A system for large-scale machine learning}.
\newblock In {\em Symposium on Operating Systems Design and Implementation
  (OSDI)}, volume~16, pages 265--283, 2016.

\bibitem{arora2018understanding}
R.~Arora, A.~Basu, P.~Mianjy, and A.~Mukherjee.
\newblock Understanding deep neural networks with rectified linear units.
\newblock In {\em International Conference on Learning Representations}, 2018.

\bibitem{borrelli2017predictive}
F.~Borrelli, A.~Bemporad, and M.~Morari.
\newblock {\em Predictive control for linear and hybrid systems}.
\newblock Cambridge University Press, 2017.

\bibitem{bouffard2012learning}
P.~Bouffard, A.~Aswani, and C.~Tomlin.
\newblock Learning-based model predictive control on a quadrotor: Onboard
  implementation and experimental results.
\newblock In {\em IEEE Int. Conf. on Robotics and Automation (ICRA)}, pages
  279--284, 2012.

\bibitem{boyd2004convex}
S.~Boyd and L.~Vandenberghe.
\newblock {\em Convex optimization}.
\newblock Cambridge University Press, 2004.

\bibitem{Chen2018ApproximatingEM}
S.~Chen, K.~Saulnier, N.~Atanasov, D.~D. Lee, V.~Kumar, G.~J. Pappas, and
  M.~Morari.
\newblock Approximating explicit model predictive control using constrained
  neural networks.
\newblock In {\em American Control Conference}, pages 1520--1527, 2018.

\bibitem{diamond2016cvxpy}
S.~Diamond and S.~Boyd.
\newblock {CVXPY}: {A} {P}ython-embedded modeling language for convex
  optimization.
\newblock {\em Journal of Machine Learning Research}, 17(83):1--5, 2016.

\bibitem{erez2013integrated}
T.~Erez, K.~Lowrey, Y.~Tassa, V.~Kumar, S.~Kolev, and E.~Todorov.
\newblock An integrated system for real-time model predictive control of
  humanoid robots.
\newblock In {\em IEEE Int. Conf. on Humanoid Robots}, pages 292--299, 2013.

\bibitem{Ferreau2008}
H.~Ferreau, H.~Bock, and M.~Diehl.
\newblock {An online active set strategy to overcome the limitations of
  explicit MPC}.
\newblock {\em International Journal of Robust and Nonlinear Control},
  18(8):816--830, 2008.

\bibitem{ferreau2014qpoases}
H.~J. Ferreau, C.~Kirches, A.~Potschka, H.~G. Bock, and M.~Diehl.
\newblock {qpOASES: A parametric active-set algorithm for quadratic
  programming}.
\newblock {\em Mathematical Programming Computation}, 6(4):327--363, 2014.

\bibitem{sqopt77}
P.~E. Gill, W.~Murray, M.~A. Saunders, and E.~Wong.
\newblock User's guide for {SQOPT 7.7}: Software for large-scale linear and
  quadratic programming.
\newblock Technical report, Department of Mathematics, UCSD, 2018.

\bibitem{Goodfellow-et-al-2016}
I.~Goodfellow, Y.~Bengio, and A.~Courville.
\newblock {\em Deep Learning}.
\newblock MIT Press, 2016.

\bibitem{MPT3}
M.~Herceg, M.~Kvasnica, C.~Jones, and M.~Morari.
\newblock {Multi-Parametric Toolbox 3.0}.
\newblock In {\em European Control Conference}, pages 502--510, 2013.

\bibitem{hertneck2018learning}
M.~Hertneck, J.~K{\"o}hler, S.~Trimpe, and F.~Allg{\"o}wer.
\newblock Learning an approximate model predictive controller with guarantees.
\newblock {\em IEEE Control Systems Letters}, 2(3):543--548, 2018.

\bibitem{john2008implementation}
E.~John and E.~A. Y{\i}ld{\i}r{\i}m.
\newblock Implementation of warm-start strategies in interior-point methods for
  linear programming in fixed dimension.
\newblock {\em Computational Optimization and Applications}, 41(2):151--183,
  2008.

\bibitem{jones2010polytopic}
C.~N. Jones and M.~Morari.
\newblock Polytopic approximation of explicit model predictive controllers.
\newblock {\em IEEE Trans. on Auto. Control}, 55(11):2542--2553, 2010.

\bibitem{kingma2014adam}
D.~P. Kingma and J.~Ba.
\newblock Adam: {A} method for stochastic optimization.
\newblock In {\em International Conference on Learning Representations}, 2015.

\bibitem{klauvco2019machine}
M.~Klau{\v{c}}o, M.~Kal{\'u}z, and M.~Kvasnica.
\newblock Machine learning-based warm starting of active set methods in
  embedded model predictive control.
\newblock {\em Engineering Applications of Artificial Intelligence}, 77:1--8,
  2019.

\bibitem{kvasnica2012clipping}
M.~Kvasnica and M.~Fikar.
\newblock {Clipping-based complexity reduction in explicit MPC}.
\newblock {\em IEEE Trans. on Auto. Control}, 57(7):1878--1883, 2012.

\bibitem{liu2017search}
S.~Liu, N.~Atanasov, K.~Mohta, and V.~Kumar.
\newblock Search-based motion planning for quadrotors using linear quadratic
  minimum time control.
\newblock In {\em IEEE/RSJ Int. Conf. on Intelligent Robots and Systems}, pages
  2872--2879, 2017.

\bibitem{lucia2018deep}
S.~Lucia and B.~Karg.
\newblock A deep learning-based approach to robust nonlinear model predictive
  control.
\newblock {\em IFAC-PapersOnLine}, 51(20):511--516, 2018.

\bibitem{qin2003survey}
S.~J. Qin and T.~A. Badgwell.
\newblock A survey of industrial model predictive control technology.
\newblock {\em Control engineering practice}, 11(7):733--764, 2003.

\bibitem{richter2018bayesian}
C.~Richter, W.~Vega-Brown, and N.~Roy.
\newblock Bayesian learning for safe high-speed navigation in unknown
  environments.
\newblock In {\em Robotics Research}, volume~2, pages 325--341. Springer, 2018.

\bibitem{ross2011reduction}
S.~Ross, G.~Gordon, and D.~Bagnell.
\newblock A reduction of imitation learning and structured prediction to
  no-regret online learning.
\newblock In {\em Int. Conf. on Artificial Intelligence and Statistics}, pages
  627--635, 2011.

\bibitem{sobol1967distribution}
I.~M. Sobol'.
\newblock On the distribution of points in a cube and the approximate
  evaluation of integrals.
\newblock {\em USSR Computational Mathematics and Mathematical Physics}, 7(4),
  1967.

\bibitem{summers2011multiresolution}
S.~Summers, C.~N. Jones, J.~Lygeros, and M.~Morari.
\newblock A multiresolution approximation method for fast explicit model
  predictive control.
\newblock {\em IEEE Transactions on Automatic Control}, 56(11):2530--2541,
  2011.

\bibitem{vempala2005geometric}
S.~Vempala.
\newblock Geometric random walks: a survey.
\newblock {\em Combinatorial \& computational geometry}, 52:577--616, 2005.

\bibitem{vichik2014solving}
S.~Vichik and F.~Borrelli.
\newblock Solving linear and quadratic programs with an analog circuit.
\newblock {\em Computers \& Chemical Engineering}, 70:160--171, 2014.

\bibitem{wang2010fast}
Y.~Wang and S.~Boyd.
\newblock Fast model predictive control using online optimization.
\newblock {\em IEEE Trans. Control Syst. Technol.}, 18(2):267--278, 2010.

\bibitem{Watterson2015SafeRH}
M.~Watterson and V.~Kumar.
\newblock Safe receding horizon control for aggressive mav flight with limited
  range sensing.
\newblock {\em IEEE/RSJ Int. Conf. on Intelligent Robots and Systems}, pages
  3235--3240, 2015.

\bibitem{wright1999numerical}
S.~Wright and J.~Nocedal.
\newblock Numerical optimization.
\newblock {\em Springer Science}, 35(67-68):7, 1999.

\bibitem{zeilinger2011real}
M.~N. Zeilinger, C.~N. Jones, and M.~Morari.
\newblock Real-time suboptimal model predictive control using a combination of
  explicit mpc and online optimization.
\newblock {\em IEEE Transactions on Automatic Control}, 56(7):1524--1534, 2011.

\bibitem{zhang2019safe}
X.~Zhang, M.~Bujarbaruah, and F.~Borrelli.
\newblock Safe and near-optimal policy learning for model predictive control
  using primal-dual neural networks.
\newblock {\em American Control Conference}, pages 354--359, 2019.

\end{thebibliography}
